\documentclass[journal]{IEEEtran}
\usepackage{graphicx}
\usepackage{amssymb}
\usepackage{amsmath}
\usepackage{amsthm}
\usepackage{array}
\usepackage{multirow}
\usepackage{rotating}
\usepackage{color}

\theoremstyle{plain}
\newtheorem{myTheo}{Theorem}

\theoremstyle{definition}

\theoremstyle{remark}

\newcommand{\abs}[1]{\left \vert #1 \right \vert}
\newcommand{\B}[1]{\mathbb{#1}}
\newcommand{\C}[1]{\mathcal{#1}}

\makeatletter

\newcommand{\Rmnum}[1]{\expandafter\@slowromancap\romannumeral #1@}
\makeatother

\begin{document}

\title{Testing Scenario Library Generation for Connected and Automated Vehicles, Part II: Case Studies}

\author{Shuo~Feng,
	Yiheng~Feng,
	Haowei~Sun,
	Shao~Bao,
	Yi~Zhang,~\IEEEmembership{Member,~IEEE}
	and~Henry~X.~Liu,~\IEEEmembership{Member,~IEEE}
	
	\thanks{This work was supported by USDOT Center for Connected and Automated Transportation at the University of Michigan, Ann Arbor. The authors would like to thank Dr. Ding Zhao of Carnegie Mellon University for providing data in the cut-in case study. \emph{(Corresponding author: Henry X. Liu)}}
	\thanks{S. Feng and Y. Zhang are with the Department of Automation, Tsinghua University, Beijing 100084, China. S. Feng is also a visiting Ph.D. student of the Department of Civil and Environmental Engineering, University of Michigan, Ann Arbor, MI, 48109, USA. (e-mail: s-feng14@mails.tsinghua.edu.cn; zhyi@mail.tsinghua.edu.cn).}
	\thanks{Y. Feng and B. Shan are with the University of Michigan Transportation Research Institude, 2901 Baxer Rd, Ann Arbor, MI 48109, USA. (e-mail: yhfeng, shanbao@umich.edu)}
	\thanks{H. Sun and H. X. Liu are with the Department of Civil and Environmental Engineering, University of Michigan, United States. (e-mail: haoweis, henryliu@umich.edu).}
}

\markboth{IEEE Transactions on Intelligent Transaportation Systems,~Vol.~, No.~, ~2020}
{Shell \MakeLowercase{\textit{et al.}}: Bare Demo of IEEEtran.cls for IEEE Journals}

\maketitle

\begin{abstract}
Testing scenario library generation (TSLG) is a critical step for the development and deployment of connected and automated vehicles (CAVs). In Part I of this study, a general methodology for TSLG is proposed, and theoretical properties are investigated regarding the accuracy and efficiency of CAV evaluation. This paper aims to provide implementation examples and guidelines, and to enhance the proposed methodology under high-dimensional scenarios. Three typical cases, including cut-in, highway-exit, and car-following, are designed and studied in this paper. For each case, the process of library generation and CAV evaluation is elaborated. To address the challenges brought by high dimensionality, the proposed methodology is further enhanced by reinforcement learning technique. For all three cases, results show that the proposed methods can accelerate the CAV evaluation process by multiple magnitudes with same evaluation accuracy, if compared with the on-road test method.

\end{abstract}


\IEEEpeerreviewmaketitle

\section{Introduction}

\IEEEPARstart{T}{esting} and evaluation is a critical step in the development and deployment of connected and automated vehicles (CAVs). In the past few years, increasing research efforts have been made to solve the testing scenario library generation (TSLG) problem \cite{li2018artificial}\cite{zhou2017reduced}\cite{PEGASUS}\cite{zhao2017accelerated}\cite{jung2007worst}\cite{zhao2018accelerated}\cite{li2019parallel} (see Part I of this study \cite{feng2019testing} for more details). However, all previous methods have limitations in either scenario types that can be handled (e.g., low-dimensional scenarios only), CAV models that can be applied (e.g., a specific CAV only), or performance metrics  that can be evaluated (e.g., safety evaluation only). 

To overcome these limitations, in Part I of this study \cite{feng2019testing}, we propose a general method for the TSLG problem for different scenario types, CAV models, and performance metrics. 
Testing scenario is evaluated by a newly proposed measure, scenario criticality, which can be computed as a combination of maneuver challenge and exposure frequency. The new measure is fundamentally different from most existing studies, which usually overvalue the worst-case scenarios  \cite{jung2007worst}\cite{zhao2018accelerated}. {In our proposed method, scenarios with higher occurrence probability in the real-world and higher maneuver challenges will have higher priority for CAV evaluation.}

Part I paper lays out the methodological foundation and proves the statistical accuracy and efficiency theoretically. To implement the proposed method, however, there exist several gaps: 

First, although the proposed framework is generic, some sub-problems vary case by case, e.g., auxiliary objective function design, naturalistic driving data analysis, and surrogate model construction. Carefully selected case studies will provide examples for the implementation of the overall framework. 

Second, it is significant to show the ability of the proposed methods in handling different performance metrics. Most existing studies focus only on safety evaluation, which is essential but insufficient for a deployable CAV. Besides safety, functionality is another important performance metric, which shows the CAV's ability to complete driving tasks in a timely manner. Designing and implementing testing scenarios for functionality evaluation are necessary. 

Finally, applying the proposed methods directly to high-dimensional cases can be problematic, as the computational complexity of critical scenario searching increases exponentially with the increase of dimensionality. However, most of the driving scenarios are naturally high-dimensional. So how to deal with high dimensional testing scenarios becomes an important issue that needs to be addressed. We should note that most existing studies also suffer from the ``curse of dimensionality''. For example, the PEGASUS project \cite{PEGASUS} applied an exhaustive searching method to find all scenarios, which is impossible for high-dimensional scenarios. The accelerated evaluation method proposed in \cite{zhao2017accelerated} also has difficulty in calibrating importance functions, where computational complexity grows exponentially with the dimensionality. 

This paper aims to fill in these gaps. Three common testing cases, including cut-in, highway exit, and car-following, will be discussed in this paper (see Fig. \ref{fig_Case}). The cut-in case illustrates each step of the scenario library generation process and our evaluation framework. Because the cut-in case is low dimensional, it is convenient to visualize the results and help readers better understand the proposed methods. The highway exit case focuses on the functionality evaluation. Compared with safety evaluation, the major difference lies in the design of auxiliary objective function for the library generation, i.e., how to quantify the maneuver challenge regarding functionality. The car-following case is designed to show the ability of the proposed methods under high-dimensional scenarios. To this end, the proposed methods in Part I are enhanced by reinforcement learning (RL) techniques. The RL-enhanced method shows the powerful ability of the  framework proposed in Part I in handling high-dimensional scenarios. For all three cases, results show that the proposed framework can effectively generate the testing scenario libraries and  accelerate the CAV evaluation process by multiple magnitudes with same evaluation accuracy, if compared with the on-road test method.

\begin{figure}[h!]
	\centering
	\begin{minipage}{.3\linewidth}
		\includegraphics[width=1\textwidth]{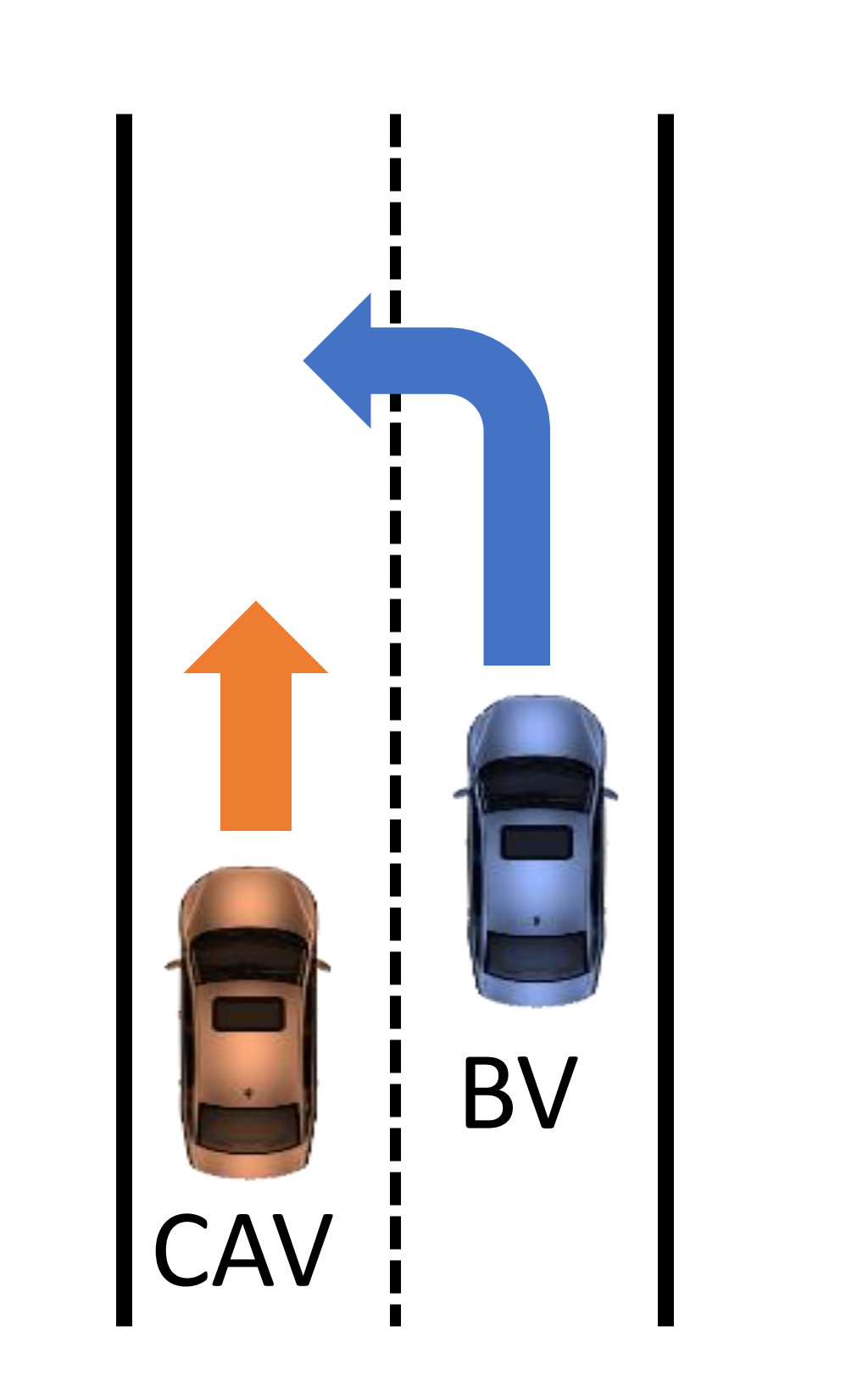}
		\centerline{(a)}
	\end{minipage}
	\begin{minipage}{.33\linewidth}
		\includegraphics[width=1\textwidth]{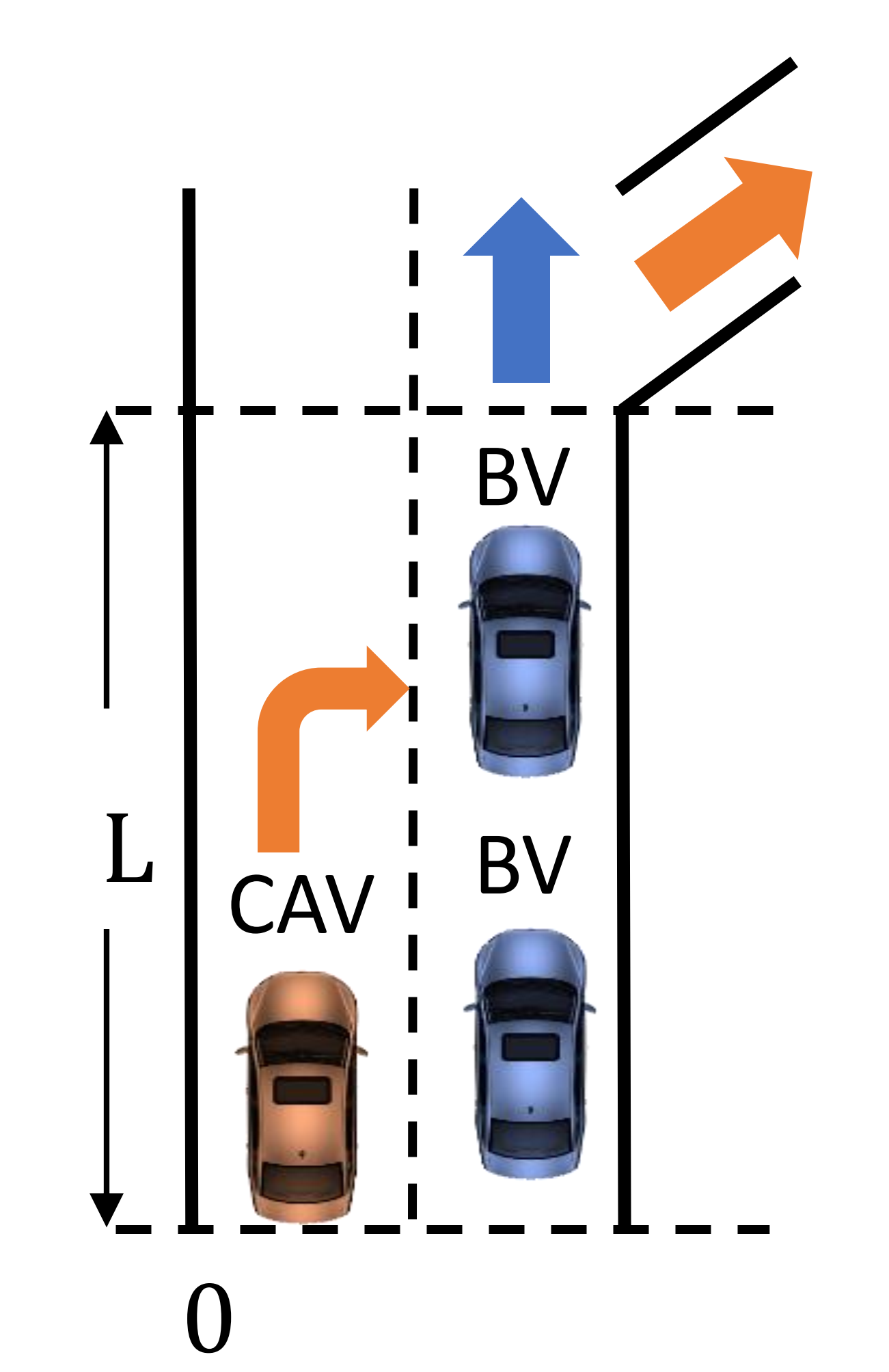}
		\centerline{(b)}
	\end{minipage}
	\begin{minipage}{.28\linewidth}
		\includegraphics[width=1\textwidth]{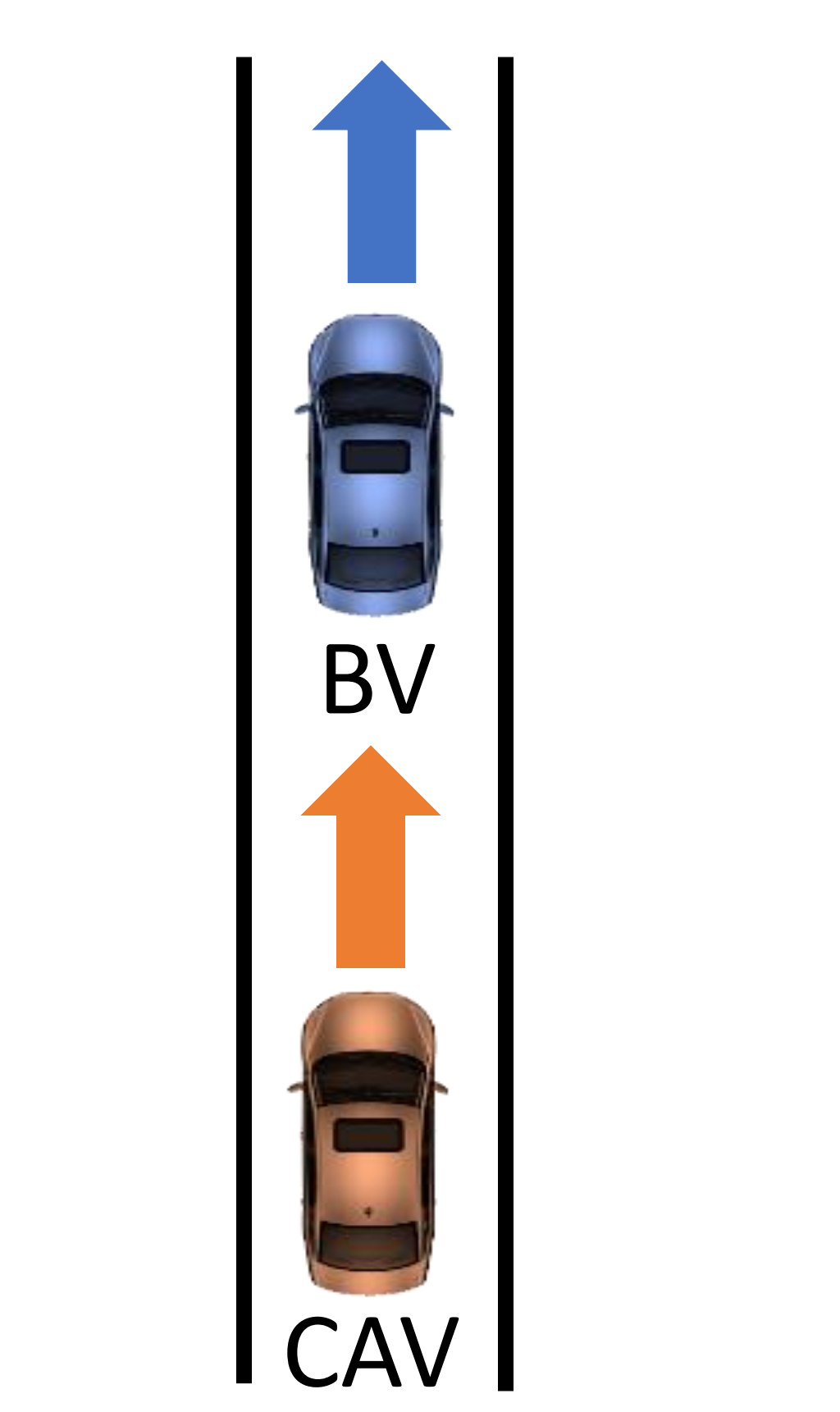}
		\centerline{(c)}
	\end{minipage}
	\caption{An illustration of the three cases: (a) cut-in, (b) highway exit, and (c) car-following. BV denotes a background vehicle.} 
	\label{fig_Case}
\end{figure}

The rest of the paper is organized as follows. For the convenience of readers, Section II briefly revisits the proposed method in Part I. Section III studies the cut-in case for safety evaluation. Section IV studies the highway exit case for functionality evaluation. In Section V, RL-enhanced method is developed for the high-dimensional car-following case. The major advantages and limitations of the method are discussed in Section VI. Finally, Section VII concludes the paper.

\section{Revisit the Proposed Method in Part I}
For the convenience of readers, in this section, we briefly revisit the proposed method in Part I \cite{feng2019testing} including problem formulation, library generation, and CAV evaluation. Notations of related variables are listed in Table \ref{tab_notation}. 

\linespread{1.2}
\begin{table}
	\centering
	\footnotesize
	\setlength{\abovecaptionskip}{1pt}
	\setlength{\belowcaptionskip}{3pt}	
	\caption{Notations of variables. }
	\label{tab_notation}
	\begin{tabular}{m{1.5cm}<{\centering}p{6.5cm}}
		\hline
		\multicolumn{1}{c}{\bfseries Variable } &   \multicolumn{1}{c}{ \bfseries Notation}   \\ \hline
		$\theta$  & \multicolumn{1}{m{6.5cm}}{Pre-determined parameters of scenarios in the operational design domain.}\\ \hline
		$x$  & \multicolumn{1}{m{6.5cm}}{Decision variables of testing scenarios.}\\ \hline
		$A$ & \multicolumn{1}{m{6.5cm}}{Event of interest (e.g., accident) with a CAV model.}\\ \hline
		$S$ & \multicolumn{1}{m{6.5cm}}{Event of interest (e.g., accident) with a surrogate model.}\\ \hline
		$\B{X}$ & \multicolumn{1}{m{6.5cm}}{Feasible set of the decision variables.} \\\hline
		$\Phi$ & \multicolumn{1}{m{6.5cm}}{Set of decision variable vector of critical testing scenarios.} \\\hline
		$V(x|\theta)$ & \multicolumn{1}{m{6.5cm}}{Criticality value of a scenario determined by $x$ and $\theta$.} \\ \hline
		$N(\B{X}), N(\Phi)$ & \multicolumn{1}{m{6.5cm}}{Total number of  scenarios of the set $\B{X}$, $\Phi$.} \\ \hline
		$\bar{P}(x_i | \theta)$ &\multicolumn{1}{m{6.5cm}}{Probability of sampling the scenario $x_i$ in the generated library with pre-determined parameters $\theta$.} \\ \hline
		$\epsilon$ & \multicolumn{1}{m{6.5cm}}{Exploration probability of $\epsilon$-greedy sampling policy.} \\\hline
		$\hat{P}(A|\theta)$ &\multicolumn{1}{m{6.5cm}}{Estimated probability of the event $A$ with pre-determined parameters $\theta$.} \\ \hline 
		$n$ &\multicolumn{1}{m{6.5cm}}{Total number of sampled testing scenarios.} \\ \hline
		$\color{black}R, \dot{R}$ & \multicolumn{1}{m{6.5cm}}{\color{black}Range and range rate at the cut-in moment between the background vehicle and test CAV.} \\\hline
		$\color{black}R(t), \dot{R}(t)$ & \multicolumn{1}{m{6.5cm}}{\color{black}Range and range rate at time $t$ between the background vehicle and test CAV.} \\\hline
		$\color{black}\omega$ & \multicolumn{1}{m{6.5cm}}{\color{black}Weight parameter.} \\\hline
		$\color{black}d(x, \Omega)$ & \multicolumn{1}{m{6.5cm}}{\color{black}Normalized distance between scenario $x$ and a high exposure frequency zone $\Omega$.} \\\hline
		$\color{black}W$ & \multicolumn{1}{m{6.5cm}}{\color{black}Normalization factor.} \\\hline
	\end{tabular}
\end{table}
\linespread{1.0}

\subsection{Problem Formulation}
The goal of the proposed method is to generate a set of critical scenarios, which can be used to evaluate CAVs for certain performance indices. If an event of interest with CAVs is denoted as $A$ (e.g., accident event), the performance of CAVs can be quantitatively evaluated by its occurrence probability (e.g., accident rate):
\begin{eqnarray}
\label{eq_PA1}
P(A|\theta) = \sum_{x\in\B{X}} P(A|x,\theta)P(x|\theta),
\end{eqnarray}
where $x$ denotes decision variables of scenarios, $\B{X}$ denotes the feasible set determined by the operational design domain (ODD), and $\theta$ denotes the pre-determined parameters under the ODD, such as road type, number of lanes, weather conditions, etc.

Essentially the on-road test is to measure $P(A|\theta)$ in a naturalistic driving environment as
\begin{eqnarray}
\label{eq_MCM}
P(A|\theta) &&\approx \frac{1}{n} \sum_{i=1}^{n}P(A|x_i,\theta), x_i \sim P(x|\theta), \\
&&\approx \frac{m}{n}, \nonumber
\end{eqnarray}
where $n$ denotes the total number of tests and $m$ denotes the number of occurrence for event $A$. Here the scenario variables $x_i$ follow the distribution from naturalistic driving data (NDD), i.e., $x_i \sim P(x|\theta)$. In this paper, either on-road tests or simulation of on-road tests with naturalistic driving environment is referred as NDD evaluation. Because the event $A$ in NDD evaluation is usually rare, the required number of tests is intolerably large with reasonable precision \cite{kalra2016driving}.

To mitigate this issue, importance sampling techniques were applied by \cite{zhao2017accelerated} as
\begin{eqnarray}
\label{eq_IS}
P(A|\theta) &&= \sum_{x\in\B{X}} \frac{P(A|x, \theta)P(x|\theta)}{q(x)}q(x), \\
&&\approx \frac{1}{n} \sum_{i=1}^{n}\frac{P(A|x_i, \theta)P(x_i|\theta)}{q(x_i)}, x_i \sim q(x), \nonumber
\end{eqnarray}
where $q(x)$ denotes an importance function. According to importance sampling theory \cite{owen2013monte}, to obtain a certain estimation precision, the required number of tests ($n$) is determined by the importance function, and how to construct a proper importance function remains a challenge. A properly constructed importance function is essential to achieving the goal of library generation. 

\subsection{Library Generation}
The basic idea of library generation is to define the criticality of scenarios and search critical scenarios to construct the library. The criticality of a scenario is defined as
\begin{eqnarray}
\label{eq_Value}
V(x|\theta) \overset{\rm def}{=} P(S|x, \theta) P(x|\theta), 
\end{eqnarray}
where $S$ denotes the event of interest with a surrogate model (SM) of CAVs, and $P(S|x, \theta)$ is the occurrence probability of $S$ in scenario $(x, \theta)$. As shown in Eq. (\ref{eq_Value}), the scenario criticality is expressed as a combination of maneuver challenge ($P(S|x, \theta)$) and exposure frequency ($P(x|\theta)$). Since too many local optimal solutions exist in the criticality function, naive search for the critical scenarios is inefficient. To solve this issue, the multi-start optimization and seed-fill based searching method is applied, where an auxiliary objective function is designed to provide searching directions. The SM and the auxiliary objective function will be discussed case by case.

\subsection{CAV Evaluation}
After the generation of library, testing scenarios are sampled from the library with $\varepsilon$-greedy policy, and the performance index ($P(A|\theta)$) is estimated based on the testing results. A minimal number of tests is required for achieving certain estimation precision.

The sampling distribution with $\varepsilon$-greedy policy is derived as
\begin{eqnarray}
\label{eq_Prob_Sampling_2}
\bar{P}(x_i|\theta) = \left\{
\begin{array}{ll}
(1-\epsilon)V(x_i|\theta)/W, &x_i \in \Phi\\
\epsilon / (N(\B{X}) - N(\Phi)), &x_i \in \B{X} \backslash \Phi
\end{array}
\right.
\end{eqnarray}
where $N(\B{X})$ denotes the total number of feasible scenarios, $\Phi$ denotes the set of critical scenarios, the selection of $\epsilon$ is theoretically analyzed (see \textcolor{black}{Corollary 1} in Part I paper), and $W$ is a normalization factor as
\begin{eqnarray}
\label{eq_w}
W = \sum_{x_i \in \Phi} V(x_i|\theta).
\end{eqnarray}

After testing the CAV with sampled scenarios, the index $P(A|\theta)$ can be estimated as
\begin{eqnarray}
\label{eq_Id_Em}
\hat{P}(A|\theta) \overset{\rm def}{=}  \frac{1}{n} \sum_{i=1}^{n} \frac{P(x_i|\theta)  }{\bar{P}(x_i|\theta)}P(A|x_i, \theta) ,
\end{eqnarray}
where $P(A|x_i, \theta)$ is estimated by the testing results. 

The minimal number of tests is shown as follow. The testing process stops if the relative half-width of the estimation is less than a pre-determined threshold $\beta$ \cite{zhao2017accelerated}\cite{wasserman2013all}\cite{ross2017introductory} as
\begin{eqnarray}
\label{eq_confidence}
\frac{z_{\alpha}}{\hat{\mu}} Var(\hat{\mu}) \le \beta,
\end{eqnarray}
where $z_{\alpha}$ is a constant with the confidence level at $100(1-\alpha)\%$, $\hat{\mu}=\hat{P}(A|\theta)$ is the estimation of the index, and $Var(\hat{\mu}) = \sigma^2/n$ is the estimation variance, which decreases with increasing number of tests, i.e., $n$. 

\section{Cut-in Case Study}
Before we present the cut-in case study, a general implementation procedure for CAV testing is provided in Fig. \ref{fig_flowchart}. For all three case studies, we will present in the order of problem formulation, library generation, and CAV evaluation.

\begin{figure}[h!]
	\centering
	\color{black}
	\includegraphics[width=0.4\textwidth]{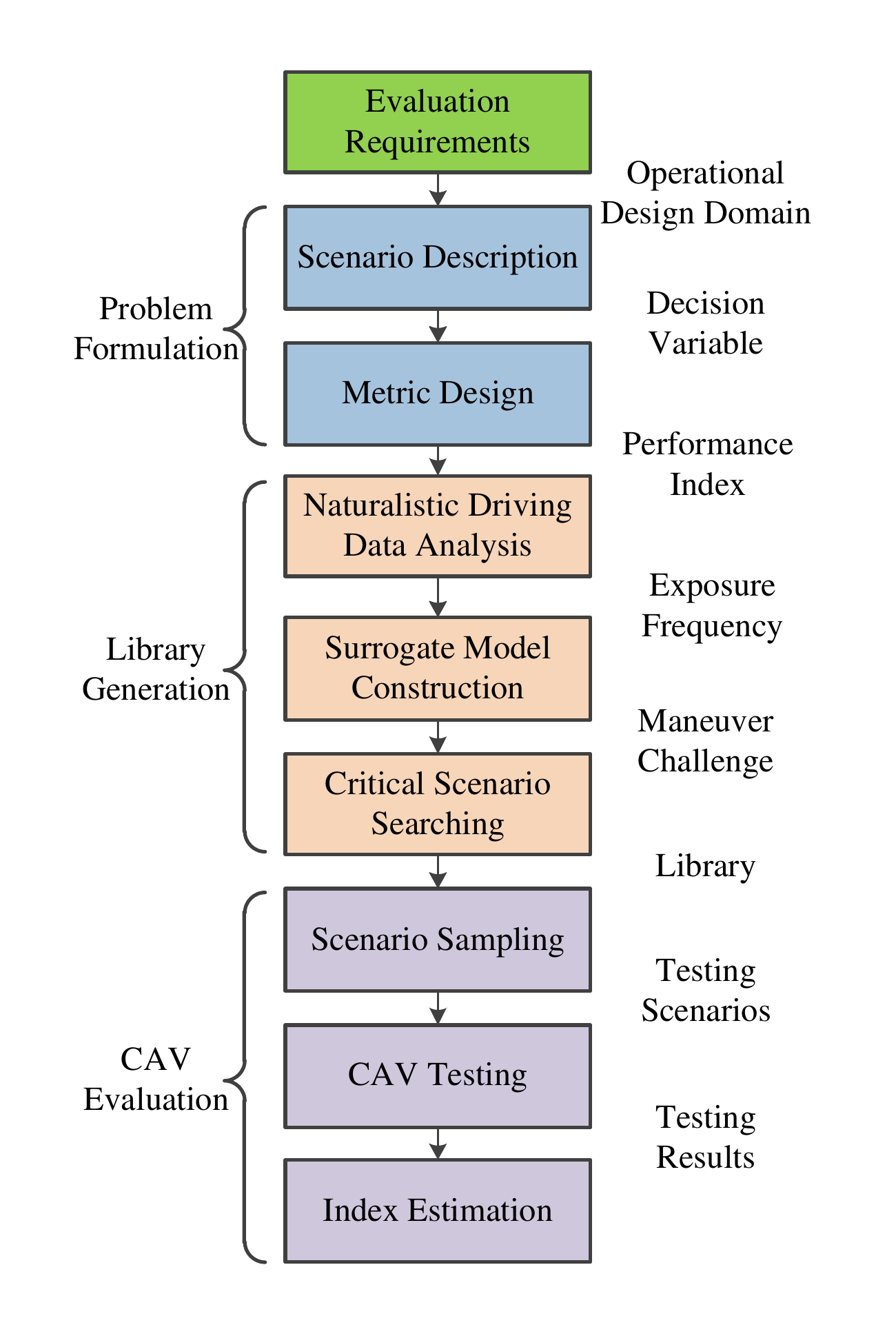}
	\caption{An illustration of the general procedure for CAV evaluation.} 
	\label{fig_flowchart}
\end{figure}

{For the cut-in case shown in Fig. \ref{fig_Case} (a), the decision variables and the performance index are formulated in Subsection III.A. In Subsection III.B, an auxiliary objective function is designed for critical scenario searching, NDD is analyzed to provide the exposure frequency, and SM is constructed to measure the maneuver frequency. With the generated library, a typical CAV is tested and evaluated in Subsection III.C.}

\subsection{Problem Formulation}
Similar to most existing studies \cite{PEGASUS}\cite{zhao2017accelerated}, the decision variables of the cut-in case are simplified as two dimensions, i.e.,
\begin{eqnarray}
x = [R, \dot{R}]^T, \nonumber
\end{eqnarray}
where $R$ and $\dot{R}$ denote the range (the longitudinal distance between the rear bumper of the preceding vehicle and the front bumper of the ego vehicle) and range rate (the longitudinal speed difference) at the cut-in moment. For this simplification, the background vehicle (BV) is assumed to keep constant velocity after the cut-in behavior, and parameters of road environments are pre-determined. All these pre-determined parameters are denoted as $\theta$.
The accident rate is utilized to measure the safety performance of CAVs in the cut-in case. The on-road test method is simulated to estimate the accident rate as a baseline. 

\subsection{Library Generation}

To implement the library generation method, three questions need to be answered specifically, i.e., auxiliary objective function design, NDD analysis, and SM construction. 

\subsubsection{Auxiliary Objective Function Design}
\label{ss_ofd}
To provide searching directions for critical scenarios, an auxiliary objective function is designed as the combination of estimated maneuver challenge and exposure frequency. 

The maneuver challenge is estimated by {\bfseries m}inimal {\bfseries n}ormalized {\bfseries p}ositive {\bfseries e}nhanced {\bfseries t}ime-{\bfseries t}o-{\bfseries c}ollision (mnpETTC). As discussed in \cite{vogel2003comparison}\cite{chen2016comparison}, ETTC is one of most widely used indices of safety evaluation for varying velocity scenarios, and it is defined as
\begin{eqnarray}
\label{eq_ETTC}
ETTC(t) = \frac{-\dot{R}(t)-\sqrt{\dot{R}^2(t)-2u_r(t)R(t)}}{u_r(t)},
\end{eqnarray}
where $R(t)$ and $\dot{R}(t)$ are the range and range rate at time $t$,  and $u_r(t)$ is the relative acceleration. Values of ETTC for different scenarios can be obtained by simulations. To make the index comparable, a normalization factor is applied, denoted as $U_I$, which is calibrated by NDD analysis. Negative values of ETTC, which denote safe situations, will be set to one. Then the minimal normalized positive ETTC (mnpETTC) can be calculated as
\begin{eqnarray}
\label{eq_mnpETTC}
mnpETTC(t) = \min_t npETTC(t),
\end{eqnarray}
where
\begin{eqnarray}
\label{eq_npETTC}
npETTC(t) = \left\{ 
\begin{matrix}
ETTC(t) / U_I, &ETTC(t)\ge0\\
1, & ETTC(t) < 0
\end{matrix}.
	\right.
\end{eqnarray}

The exposure frequency of a scenario is estimated by the distance between the scenario and a common set  (i.e., scenarios with high exposure frequency). The common set is determined by NDD analysis, and the distance is defined as
\begin{eqnarray}
\label{eq_d}
d(x,\Omega) &&= \min_{y \in \Omega} d(x,y), \\
&& = \min_{y \in \Omega} \sqrt{\frac{1}{m_d} \sum_{i=1}^{m_d} \frac{(x_i-y_i)^2}{U^2_{F,i}} }, \nonumber
\end{eqnarray}
where $\Omega$ denotes the common set, $m_d$ is the dimension of the decision variables, and $U_{F,i}$ is the normalization factor for the $i$-th dimension, which is calibrated by NDD analysis.

The auxiliary objective function for safety evaluation in the cut-in case is formulated as
\begin{eqnarray}
\label{eq_obj_cutin}
\min_x J(x) = \min_x \left( mnpETTC(x) + w \times d(x, \Omega)   \right),
\end{eqnarray}
where $w \in (0,1]$ is a balance weight. Note the goal of the auxiliary objective function is to provide searching directions, so certain roughness (e.g., caused by $w$) is reasonable and acceptable. The values of the parameters for the auxiliary objective function are listed in Table \ref{tab_cutin}.

\linespread{1.2}
\begin{table}[h!]
	\centering
	\footnotesize
	\setlength{\abovecaptionskip}{1pt}
	\setlength{\belowcaptionskip}{3pt}	
	\caption{The auxiliary objective function parameters in the cut-in case. }
	\label{tab_cutin}
	\begin{tabular}{cccc}
		\hline
		\multicolumn{1}{c}{\bfseries Parameter } &   \multicolumn{1}{c}{ \bfseries Value} &\multicolumn{1}{c}{\bfseries Parameter } &   \multicolumn{1}{c}{ \bfseries Value}  \\ \hline
		$m_d$ & 2 &$U_I$ &100 \\ 
		$U_{F,1}$ &18 &$U_{F,2}$ &20\\ 	
		$w$ &1.0 &- &-\\ \hline
		
	\end{tabular}
\end{table}
\linespread{1.0}

\subsubsection{NDD Analysis}
\label{ss_NDD}
NDD is analyzed to provide exposure frequency measurement, determine parameters of the auxiliary objective function, and calibrate the SM. 

The NDD from the Safety Pilot Model Deployment (SPMD) program at University of Michigan \cite{bezzina2014safety} is utilized for the cut-in case. The SPMD database is one of the largest databases in the world that recorded naturalistic driving behaviors over 34.9 million travel miles from 2,842 equipped vehicles in Ann Arbor, Michigan. In the database, there are 98 sedans equipped with the data acquisition system and MobilEye, which enables the measurement of both longitudinal and lateral distances between the ego vehicle, preceding vehicles, and lane markings, at a frequency of 10 Hz. By analyzing these lateral distances, cut-in events can be identified. For each cut-in event, the cut-in moment is determined by the time instant when the preceding vehicle crosses the lane marking, and the range and range rate at that moment are recorded for the NDD analysis. In this paper, the following query criteria \cite{zhao2017accelerated}\cite{gong2018evaluation} are designed to extract all cut-in events from the database: (a) the vehicles' speeds at the cut-in time belong to $(2m/s,40m/s)$; (b) the range at the cut-in time belongs to $(0.1m,90m)$. As a result, 414,770 qualified cut-in events are successfully obtained. Fig. \ref{fig_Map} shows the location distribution of the events. The exposure frequency distribution (i.e., $P(x|\theta)$) is shown in Fig. \ref{fig_Px}, where brighter color denotes higher exposure frequency. The range and range rate are discretised by $2m$ and $0.4m/s$ respectively. The NDD evaluation method is equivalently sampling testing scenarios from this probability distribution.

\begin{figure}[h!]
	\centering
	\includegraphics[width=0.45\textwidth]{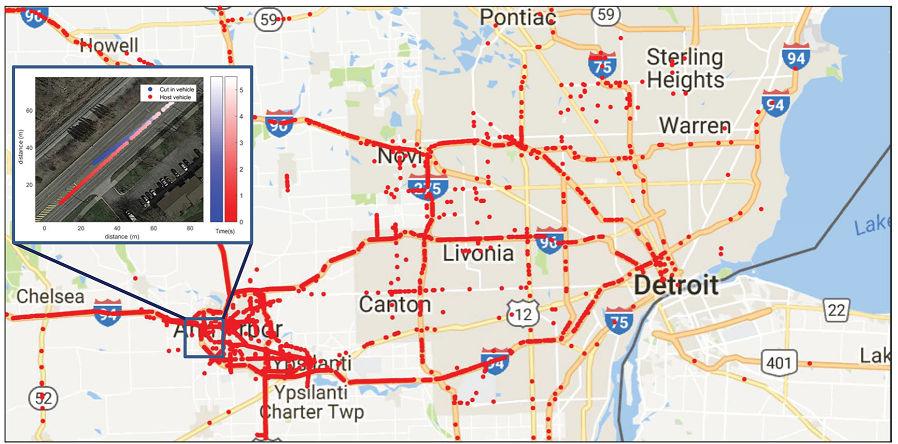}
	\caption{An illustration of the cut-in event distribution in the Safety Pilot Database \cite{gong2018evaluation}. }
	\label{fig_Map}
\end{figure}

\begin{figure}[h!]
	\centering
	\includegraphics[width=0.45\textwidth]{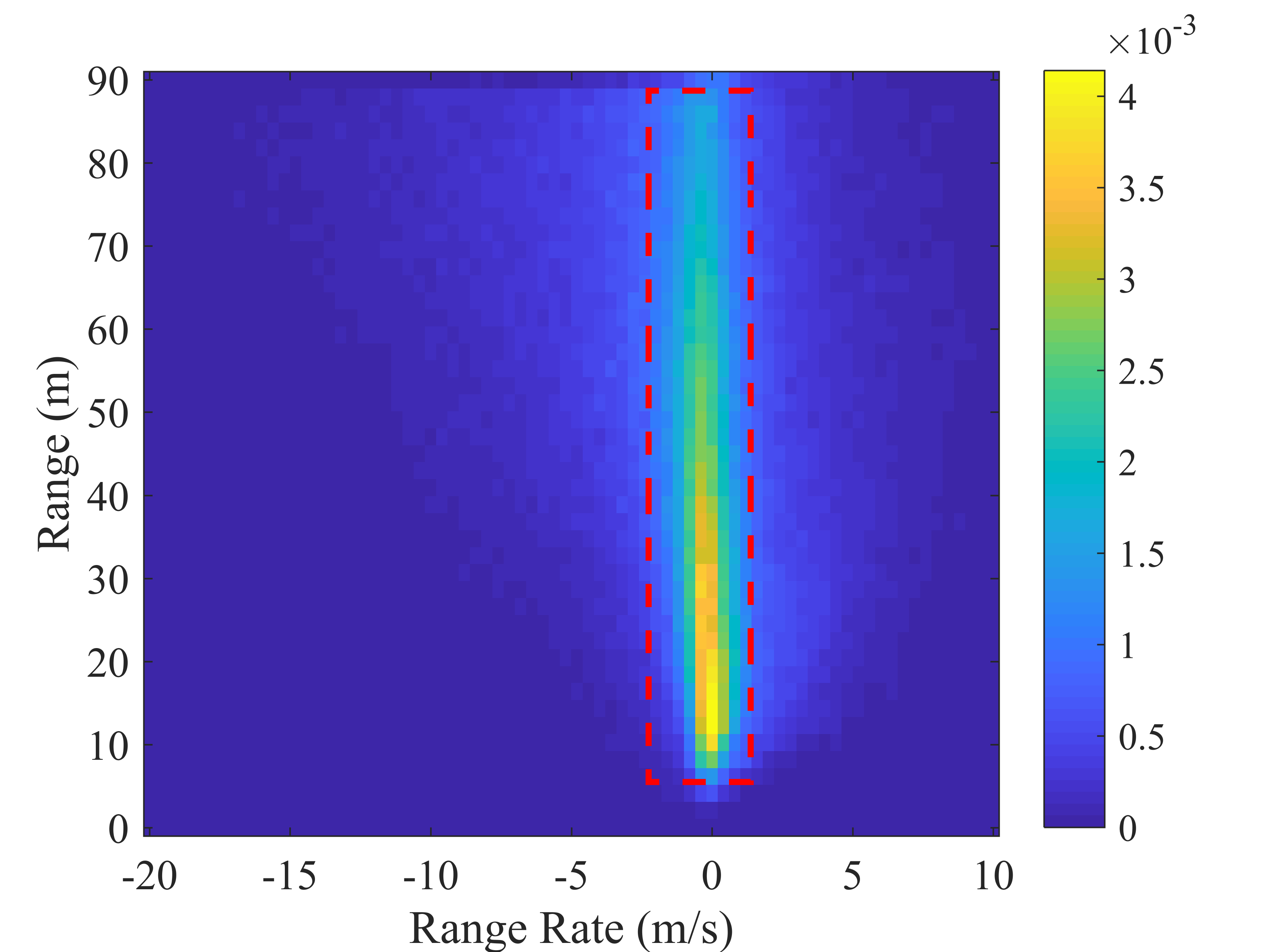}
	\caption{The exposure frequency (i.e., $P(x|\theta)$) of the cut-in range and range rate. The dashed red rectangle denotes the boundary of the common set (i.e., $\Omega$).} 
	\label{fig_Px}
\end{figure}

We now discuss how to determine the parameters in the auxiliary objective function. First, the common set $\Omega$ can be determined by finding a minimal rectangle or hyper-rectangle of scenarios with high probabilities (i.e., $P(x|\theta)>10^{-3}$). As shown in Fig. \ref{fig_Px}, the dashed red rectangle denotes the boundaries of the common set in the cut-in case (i.e., $[6, 88]$ for range and $[-2.4, 1.2]$ for range rate). For more complex scenarios, the common set can be further simplified as the most frequent scenario, e.g., $R=14, \dot{R}=0$. Second, the normalization factors are determined by the maximal distance between scenarios and the common set. For example, the maximal range rate between scenarios and the common set is smaller than 18, so the normalization factor of the range rate is set to 18. 

\subsubsection{Surrogate Model Construction}
\label{ss_SM}

SM construction is a very important step in the library generation process. It represents what we know about the generic features of CAVs. The ``generic features'' capture basic behavior of a CAV, e.g., keep safe distances with surrounding vehicles. Similar to human drivers, where different drivers have different driving habits, generic features exist among all drivers. An ideal SM should be calibrated from actual CAV driving data similar to human driving model calibration \cite{ranney1994models}. At the current stage, however, there is very little open CAV data available for public research.  Therefore, we propose to calibrate the SM based on the human driving data, i.e., NDD. 

\begin{figure}[h!]
	\centering
	\includegraphics[width=0.45\textwidth]{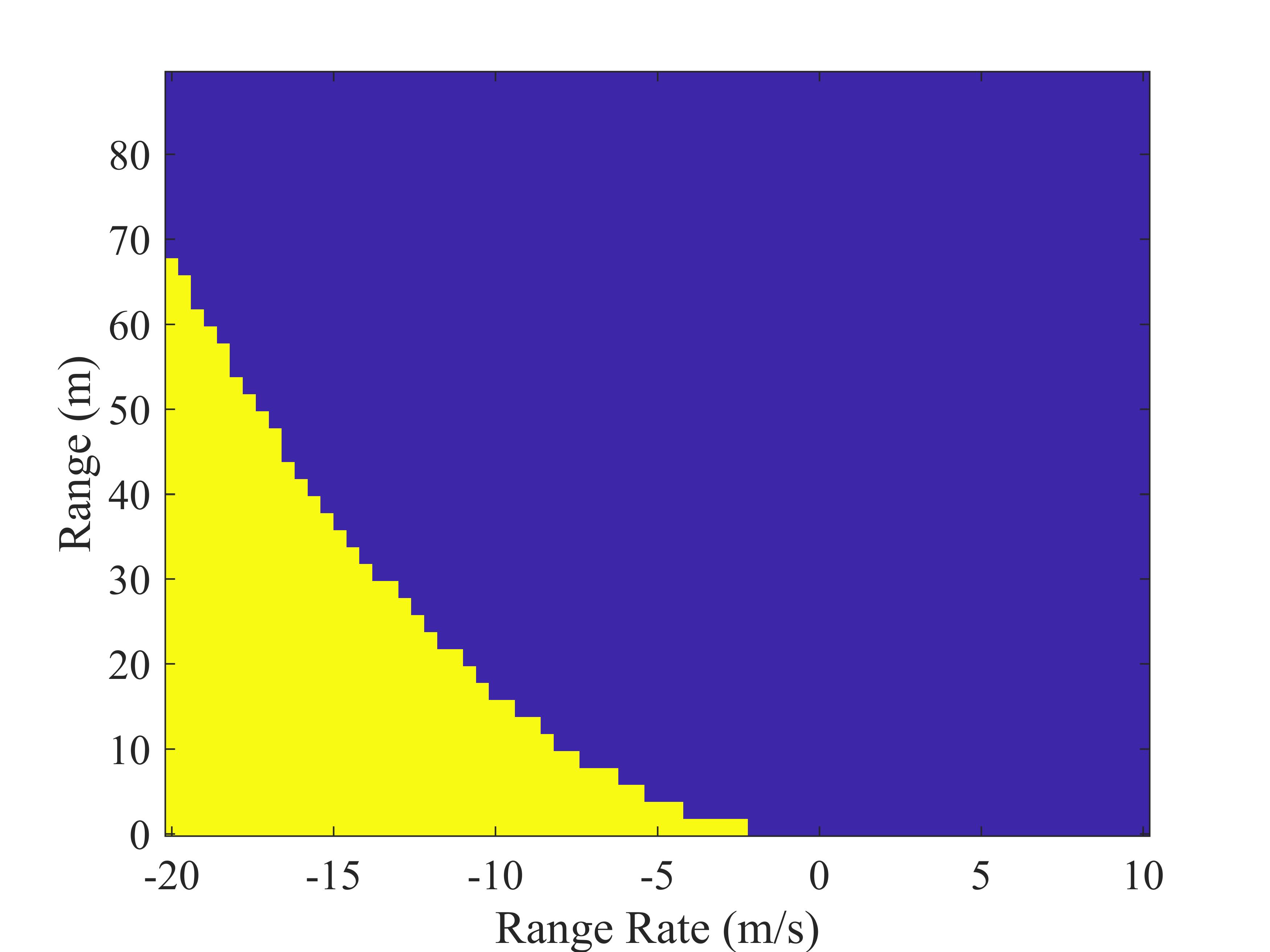}
	\caption{The safety performance of the SM, where the SM has accidents in scenarios of the yellow region.} 
	\label{fig_IDM}
\end{figure}
In this case study, a calibrated intelligent driving model (IDM) \cite{ro2018formal} is selected as the SM for the car-following behaviors after the cut-in event:
\begin{eqnarray}
\label{eq_SM_ct}
&&u(k+1) = \\
&&\alpha_{\text{IDM}} \left( 1 - \left(\frac{v(k)}{\beta_{\text{IDM}}}\right)^{c_{\text{IDM}}}  
- \left( \frac{s_{\text{IDM}}(v(k), \dot{R}(k))}{R(k)-L_{\text{IDM}}} \right)^2
\right), \nonumber
\end{eqnarray}
where $k$ denotes the discrete time step, $u$ denotes the acceleration, $\alpha_{\text{IDM}}$, $\beta_{\text{IDM}}$, $c_{\text{IDM}}$,  $L_{\text{IDM}}$ are constant parameters, and
\begin{eqnarray}
s_{\text{IDM}}(v(k), \dot{R}(k)) = s_0 + v(k)T + \frac{v(k)\dot{R}(k)}{2\sqrt{\alpha_{\text{IDM}} b_{\text{IDM}}}},
\end{eqnarray}
where $s_0$, $b_{\text{IDM}}$, and $T$ are constant parameters.
Similar to \cite{hamdar2008existing}, the constraints of acceleration and velocity are added to make the model more practical (i.e., model accident-prone behaviors) as
\begin{eqnarray}
v_{min} \le v \le v_{max}, a_{min} \le u \le a_{max}.
\end{eqnarray}
An accident event is defined as the vehicle range is smaller than a threshold, i.e., $R(t)<d_{acci}$.
The calibrated values are listed in Table \ref{tab_cutin_SM}.  Fig. \ref{fig_IDM} shows the safety performance of the selected SM, where the SM has accidents in scenarios of the yellow region.
\linespread{1.2}
\begin{table}[h!]
	\centering
	\footnotesize
	\setlength{\abovecaptionskip}{1pt}
	\setlength{\belowcaptionskip}{3pt}	
	\caption{The surrogate model parameters in the cut-in case. }
	\label{tab_cutin_SM}
	\begin{tabular}{cccc}
		\hline
		 \multicolumn{1}{c}{\bfseries Parameter } &   \multicolumn{1}{c}{ \bfseries Value} &\multicolumn{1}{c}{\bfseries Parameter } &   \multicolumn{1}{c}{ \bfseries Value}  \\ \hline
		$v_{max}$ &40 $m/s$	&$v_{min}$ &2 $m/s$ \\
		$a_{min}$ &-4 $m^2/s$ &$a_{max}$ &2 $m^2/s$ \\ 
		$\alpha_{\text{IDM}}$ &2 &$\beta_{\text{IDM}}$ &18\\
		$c_{\text{IDM}}$ &4 &$s_0$ &2\\
		$L_{\text{IDM}}$ &4 &$T$ &1\\
		$b_{\text{IDM}}$ &3 &$d_{acci}$ &1 $m$\\	 
		\hline
		
	\end{tabular}
\end{table}
\linespread{1.0}

\subsubsection{Library Generation}
\label{ss_LG}
The optimization and seed-fill based method proposed in Part I of this study \cite{feng2019testing} is applied to search for critical scenarios and construct the library.
In this case, 50 points are uniformly sampled as the initial starting points. As discussed in \textcolor{black}{Corollary 2} in Part I, the threshold of critical scenarios is determined as 
\begin{eqnarray}
\label{eq_thresh_VR_cutin}
\gamma = \frac{m}{N(\B{X})-N(\Phi)} \approx \frac{m}{N(\B{X})} = 2.9 \times 10^{-4},
\end{eqnarray}
where $m=1$, and $N(\B{X})=47 \times 76 = 3,420$. The discretization intervals of the range and range rate are $2m$ and $0.4m/s$, and the boundaries of the range and range rate are $(0,90]$ and $[-20,10]$ respectively. 

\begin{figure}[h!]
	\centering
	\includegraphics[width=0.45\textwidth]{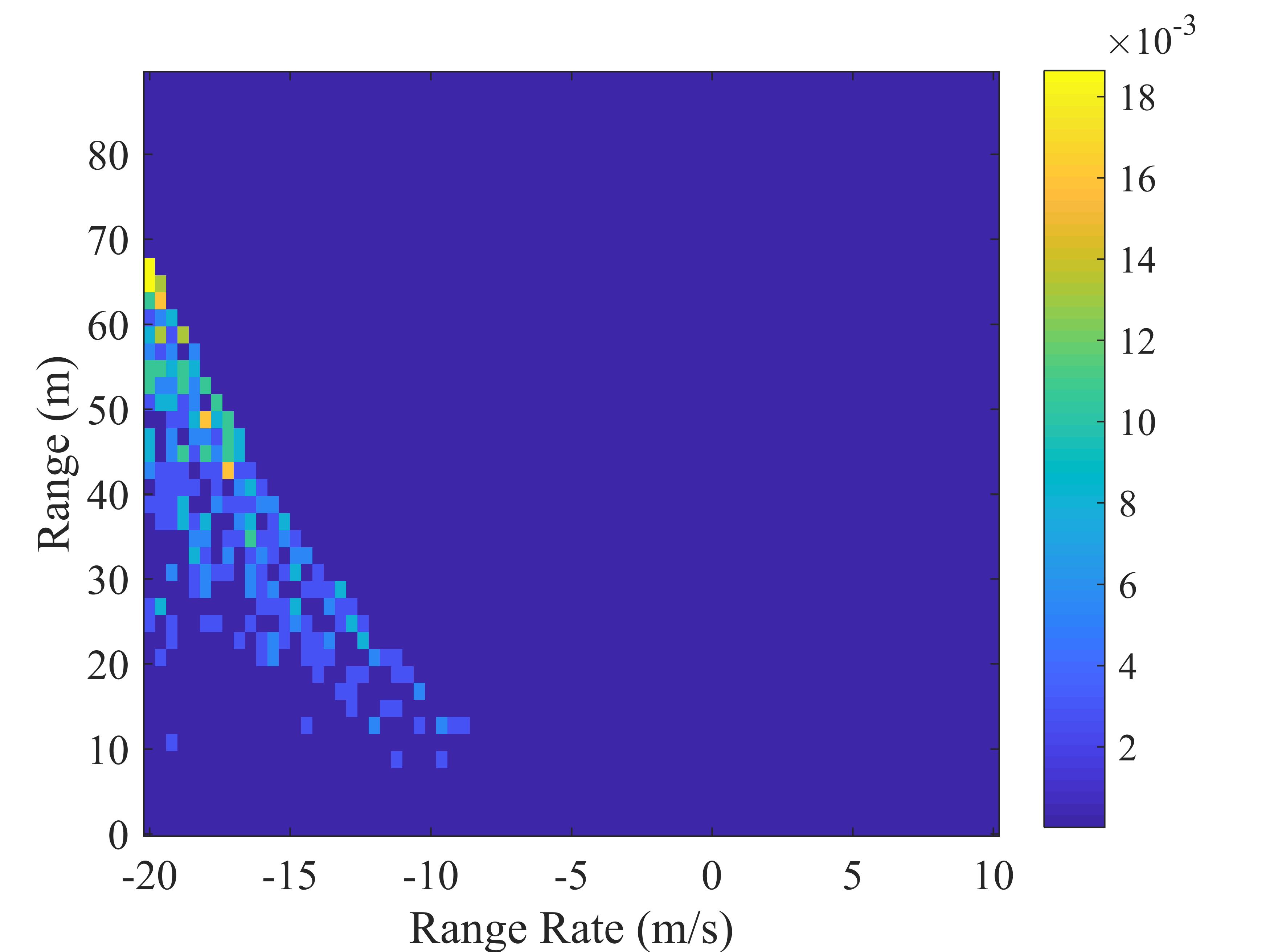}
	\caption{The generated library of the cut-in case for safety evaluation. The color denotes the testing probability of scenarios.} 
	\label{fig_lib}
\end{figure}

Fig. \ref{fig_lib} shows the obtained probability distribution after the library generation process. The color denotes the probability of a scenario, which is calculated by the normalized criticality. Compared with Fig. \ref{fig_Px}, where only exposure frequency is considered, the new distribution considers both the maneuver challenge and exposure frequency of scenarios. All scenarios with criticality values exceeding the threshold $\gamma$ will be included in the library. 
In this case, the generated library contains a total number of 184 scenarios, which is about 5.38\% of all scenarios.

\subsection{CAV Evaluation}
In this step, a specific CAV is evaluated with the generated library. For field implementations, a real CAV should be tested. In this paper, a simulated CAV model is used as a proof of concept to validate the proposed method.

The simulation model used in \cite{zhao2017accelerated}, which combines adaptive cruise control and autonomous emergency braking functions, is adopted in this case. The NDD evaluation method is applied as the baseline, where testing scenarios are sampled from the NDD distribution in Fig. \ref{fig_Px}. For the proposed method, testing scenarios are sampled from the generated library in Fig. \ref{fig_lib}. The $\epsilon$-greedy sampling policy is applied with $\epsilon=0.05$, which is determined according to \textcolor{black}{Corollary 1} in Part I paper \cite{feng2019testing}. The chosen CAV model is tested in the sampled scenarios, and the accident event is recorded.

\begin{figure}[h!]
	\centering
	\begin{minipage}{.9\linewidth}
		\includegraphics[width=1\textwidth]{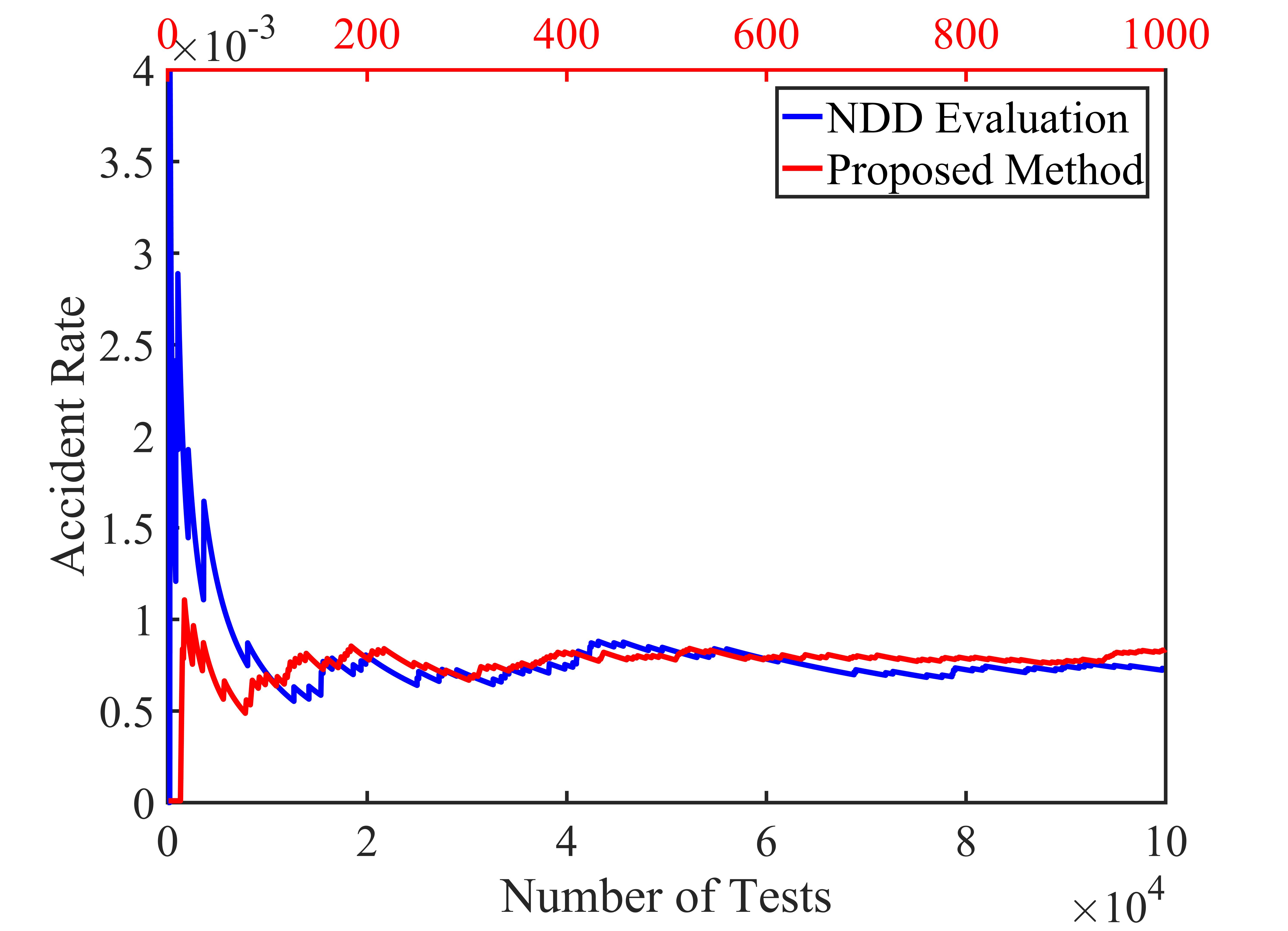}
		\centerline{(a)}
	\end{minipage}
	\begin{minipage}{.9\linewidth}
		\includegraphics[width=1\textwidth]{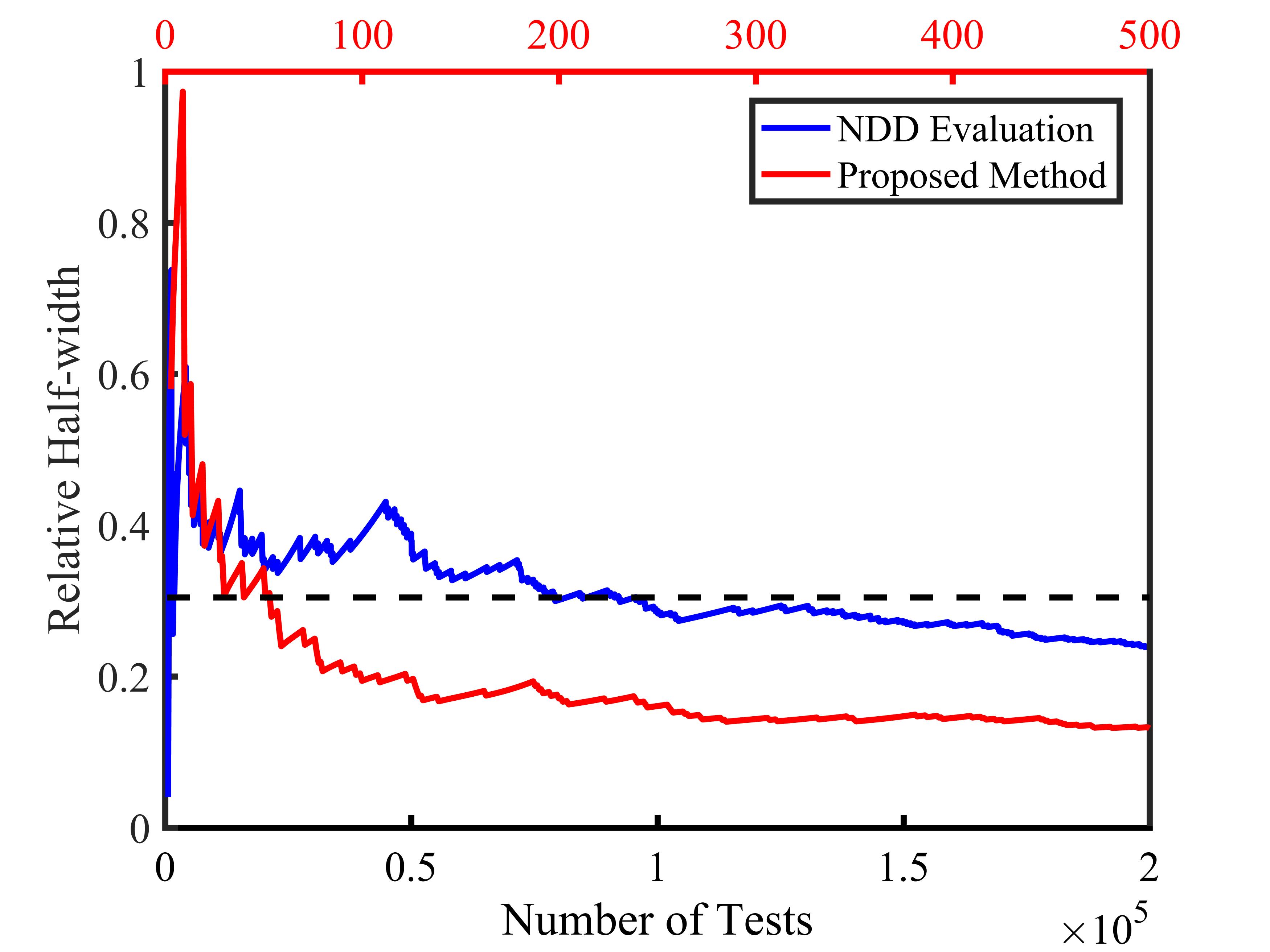}
		\centerline{(b)}
	\end{minipage}
	\caption{Results of the cut-in case: (a) estimation results of the accident rate; (b) relative half-width of the estimation results.} 
	\label{fig_result_cutin}
\end{figure}

Fig. \ref{fig_result_cutin} shows the comparison of the two evaluation methods. The blue line represents the results of the NDD evaluation method, and the bottom $x$-axis represents its number of tests. The red line represents the results of the proposed method, and the top $x$-axis represents its number of tests. As shown in Fig. \ref{fig_result_cutin}, both methods can obtain accurate estimation of the accident rate for a pre-determined relative half-width. In this paper, we set  $\alpha=0.95$ and $\beta=0.3$. Fig. \ref{fig_result_cutin} (b) shows that the proposed method achieves this confidence level after 51 tests, while the NDD evaluation  method needs $9.63 \times 10^4$ tests. The proposed method is about 1,888 times faster than the NDD evaluation method. 

\section{Highway Exit Case Study}
As shown in Fig. \ref{fig_Case} (b), for the highway exit case, the test CAV needs to make a lane change to the right and exits the highway within a certain distance. Compared with safety evaluation, the major difference of functionality evaluation lies in the design of auxiliary objective function for critical scenario searching. To this end, several new concepts are proposed, i.e., task, task solution, task solution difficulty, and task difficulty. Based on these concepts, an auxiliary objective function is designed.

\subsection{Problem Formulation}
The decision variables of the highway exit scenario include initial states of the CAV,  number of BVs, and trajectories of each BV, which is high-dimensional. To simplify the problem and focus on the functionality evaluation, the initial position and velocity of the CAV are pre-determined as $p_0$ and $v_0$, and only two BVs is considered. The two BVs will keep their initial velocity unless the distance between them is less than a threshold $d_{cf}$. In such case, the following BV will change its speed to be the same as the leading BV. As a result, the decision variables are formulated as
\begin{eqnarray}
\label{eq_x_highway}
x = [p_{0,1}, v_{0,1}, p_{0,2}, v_{0,2}]^T,
\end{eqnarray}
where $p_{0, i}$, $v_{0, i}$ denote the initial position and velocity of the $i$-th BV.
The discrete interval of time and position is chosen as $\Delta t$ and $\Delta p$ respectively. In this case study, the initial problem settings are summarized in Table \ref{tab_highway_PF}. 

\linespread{1.2}
\begin{table}[h!]
	\centering
	\footnotesize
	\setlength{\abovecaptionskip}{1pt}
	\setlength{\belowcaptionskip}{3pt}	
	\caption{The initial problem settings for the highway exit case. }
	\label{tab_highway_PF}
	\begin{tabular}{cccc}
		\hline
		\multicolumn{1}{c}{\bfseries Parameter } &   \multicolumn{1}{c}{ \bfseries Value} &\multicolumn{1}{c}{\bfseries Parameter } &   \multicolumn{1}{c}{ \bfseries Value}  \\ \hline
		$p_0$ &0 $m$ &$v_0$ & 30 $m/s$ \\ 
		$d_{cf}$ &2 $m$	&$p_{0,i}$ &$[-100, 200]$ \\
		$v_{0,i}$ &[20,40]	&$\Delta t$ &0.1 $s$ \\
		$\Delta p$ &5 $m$ &-&-\\ \hline
	\end{tabular}
\end{table}
\linespread{1.0}

\subsection{Library Generation}
The library generation methods are the same as those presented in the cut-in case. To make the paper concise, only the auxiliary objective function design for the functionality evaluation will be elaborated.

\subsubsection{Auxiliary Objective Function Design}
Similar to the cut-in case, the auxiliary objective function is composed of estimated exposure frequency and maneuver challenge. To evaluate the maneuver challenge for generic functionality, four new concepts are proposed, i.e., task, task solution, task solution difficulty, and task difficulty. The ``task'' is defined based on the functionality, e.g., exit from the highway. The ``task solution'' $f$ denotes a feasible CAV trajectory to complete the task, i.e., $f \in \B{F}$. $\B{F}$ represents the feasible set of CAV trajectories. The ``task solution difficulty'' denotes the difficulty in completing the task solution, i.e., $W(f)$, where $W(f)$ is negative and larger $W(f)$ denotes higher difficulty. Finally, the ``task difficulty''  denotes the difficulty of the task, which can be evaluated by the summation of all task solution difficulties as
\begin{eqnarray}
\label{eq_difficulty}
M_f(x) = \sum_{f\in\B{F}}W(f).
\end{eqnarray}
This definition can represent both the difficulty in finding a feasible solution to the task and completing that solution. 

\begin{figure}[h!]
	\centering
	\includegraphics[width=0.45\textwidth]{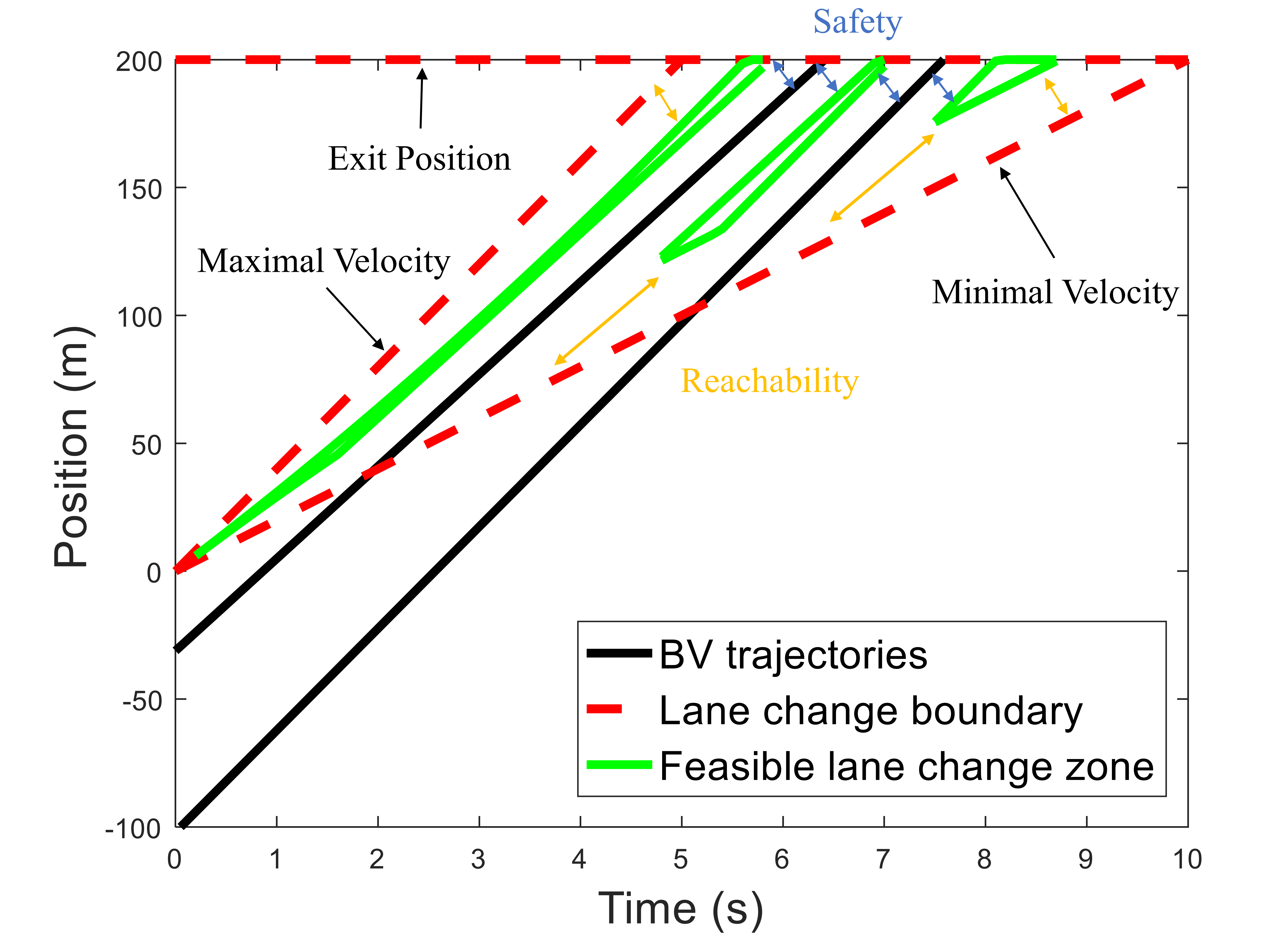}
	\caption{Illustration of the task difficulty evaluation of the highway-exit case.} 
	\label{fig_obj_highway}
\end{figure}

For the specified highway exit case, the maneuver challenge is evaluated based on the proposed concepts. The task is to make a lane change to the right before reaching the off-ramp location. The task solution is defined as a feasible lane-change point $f=(t,p)$, where $t$ is the lane-change time and $p$ is the lane-change position. The feasible lane-change zone $\B{F}$ is determined by maximal/minimal velocity $(v_{max},v_{min})$, highway exit location ($L$), safe time-to-collision gaps ($t_{min}$), and maximal/minimal acceleration ($a_{max}$, $a_{min}$).  Fig. \ref{fig_obj_highway} illustrates an example of the feasible lane change zone for a specific scenario, i.e., $x=[-25,34.5,-100,40]^T$. The initial position of the CAV is set zero. The lane change boundary (shown as the red dashed line) is determined by the maximal/minimal velocity and the off-ramp location. The feasible lane change zone (shown as the green lines), i.e., $\B{F}$, consists of three isolated zones, which are separated by the trajectories of BVs (shown as the black lines). 

For simplicity, we assume all task solutions of this case have the same task solution difficulty. Then, the task difficulty can be estimated as
\begin{eqnarray}
M_f(x) = \sum_{f\in\B{F}}W(f) = -S(\B{F}),
\end{eqnarray}
where $S(\B{F})$ denotes the area of the feasible lane-changing zone. To make the index comparable with exposure frequency, a normalization factor is applied, denoted as $U_S$, which can be obtained by the area enclosed by the lane change boundary.
Finally, the auxiliary objective function of the highway exit case is designed as
\begin{eqnarray}
\min_x J(x) = \min_x \left(  S(\B{F}) / U_S + w \times d(x, \Omega)  \right),
\end{eqnarray}
where $w$ is the weight, and $d(x,\Omega)$ can be obtained similarly as in the cut-in case (Eq. (\ref{eq_d})). The common set ($\Omega$) in this case can be constructed by most frequent scenarios. The parameter values of the auxiliary objective function are listed in Table \ref{tab_highway}.

\linespread{1.2}
\begin{table}[h!]
	\centering
	\footnotesize
	\setlength{\abovecaptionskip}{1pt}
	\setlength{\belowcaptionskip}{3pt}	
	\caption{The parameters for the highway exit case. }
	\label{tab_highway}
	\begin{tabular}{cccc}
		\hline
		\multicolumn{1}{c}{\bfseries Parameter } &   \multicolumn{1}{c}{ \bfseries Value} &\multicolumn{1}{c}{\bfseries Parameter } &   \multicolumn{1}{c}{ \bfseries Value}  \\ \hline
		$a_{max}$ &2 $m^2/s$ &$a_{min}$ &-4$m^2/s$\\ 
		$L$ &200 $m$ & $w$ &1\\ 
		$t_{min}$ &0.5 $s$ &$t$ &$[0,10]$ \\
		$p$ &$[0,200]$ &$U_S$ &500 \\ \hline
	\end{tabular}
\end{table}
\linespread{1.0}

\subsubsection{NDD Analysis}
The NDD from the Integrated Vehicle-Based Safety System (IVBSS) project is used to provide exposure frequency information \cite{sayer2011integrated}. In the IVBSS project, 108 randomly sampled drivers from different ages used sixteen Honda Accord vehicles in an unsupervised manner for a period over 40 days. 
In this paper, the exposure frequency of highway exit scenarios is determined by the car-following events of the two BVs at the rightmost lane. Query criteria are designed to extract car-following events from the database as: (1) vehicle was traveling on a highway; (2) vehicle was traveling at a speed of at least 20 $m/s$ ($\approx$45 $mph$); (3) cruise control function was not activated; (3) dry surface condition; (4) day light condition. The resulting dataset represents a total of $5\times 10^4$ car-following events and $1.47\times10^6$ points of car-following trajectories. The exposure frequency of a highway exit scenario can be estimated as
\begin{eqnarray}
P(x|\theta) = P(p_{0,1}|\theta)P( v_{0,1}, R, v_{0,2}|\theta),
\end{eqnarray}
where $R=p_{0,1}-p_{0,2}$, $P(p_{0,1}|\theta)$ denotes the initial position probability of the leading BV, which can be estimated by uniform distribution, and $P( v_{0,1}, R, v_{0,2}|\theta)$ is obtained from the distribution of car-following trajectories in the NDD.

\subsubsection{SM Construction}
The MOBIL (`minimizing overall braking induces by lane changes') model was proposed by \cite{kesting2007general} to derive human lane-changing rules for discretionary and mandatory lane changes. It provides the utility measurement method for deciding which gap has a desirable lane change position as
\begin{eqnarray}
U_{LG} = \tilde{u}-u + p_{LG}\left( \tilde{u}_{new} -u_{new} + \tilde{u}_{old} - u_{old} \right),
\end{eqnarray}
where $\tilde{u}$ denotes the new acceleration of the CAV after the lane change, $p_{LG}$ is the politeness factor, and $u_{new}, u_{old}$ denote the acceleration of the new follower and old follower respectively. As it is desirable to complete the lane change, the politeness factor is set close to zero, e.g., $p_{LG}=0.1$. To predict the CAV's trajectories before the lane-change, the Model Predictive Control (MPC) \cite{rawlings2009model} is applied, and the trajectory with higher predictive utility of lane change, i.e., $U_{LG}$, will be chosen as the solution to the task.

\subsubsection{Library Generation}
Similar to the cut-in case, a hundred points are uniformly sampled as the initial starting points for the optimization method, and the threshold of critical scenarios is determined as
\begin{eqnarray}
\gamma = \frac{1}{N(\B{X})} = 6.1 \times 10^ {-7},
\end{eqnarray}
similar to Eq. (\ref{eq_thresh_VR_cutin}). The size of total scenarios is $N(\B{X}) = n_p^2 \times n_v^2 = 1.64 \times 10^6$, where $n_p=61$ and $n_v=21$ denote the number of the feasible value of variables $p_{0, i}$ and $v_{0, i}$ respectively. 
After applying the critical scenario searching method, the testing scenario library of the highway exit case is generated. The total number of critical scenarios in the library is 1,895, which is about 0.12\% of all scenarios.

\subsection{CAV Evaluation}
The CAV lane-change model developed in \cite{nilsson2016if} is used for evaluation in this case study. Similarly, the NDD evaluation method is used as the benchmark.
In the proposed method, testing scenarios are sampled from the generated highway exit library, and events of task failures (i.e., cannot exit from the highway) are recorded. Similar to the cut-in case, the $\epsilon$-greedy sampling policy is applied with $\epsilon=0.10$. The task failure rate is estimated to measure the functionality performance of the CAV model in the highway exit case. After the estimated task failure rate converges to a certain estimation precision, the estimated task failure rate is obtained, and the evaluation process is completed, as shown in Eq. (\ref{eq_confidence}).

\begin{figure}[h!]
	\centering
	\begin{minipage}{.9\linewidth}
		\includegraphics[width=1\textwidth]{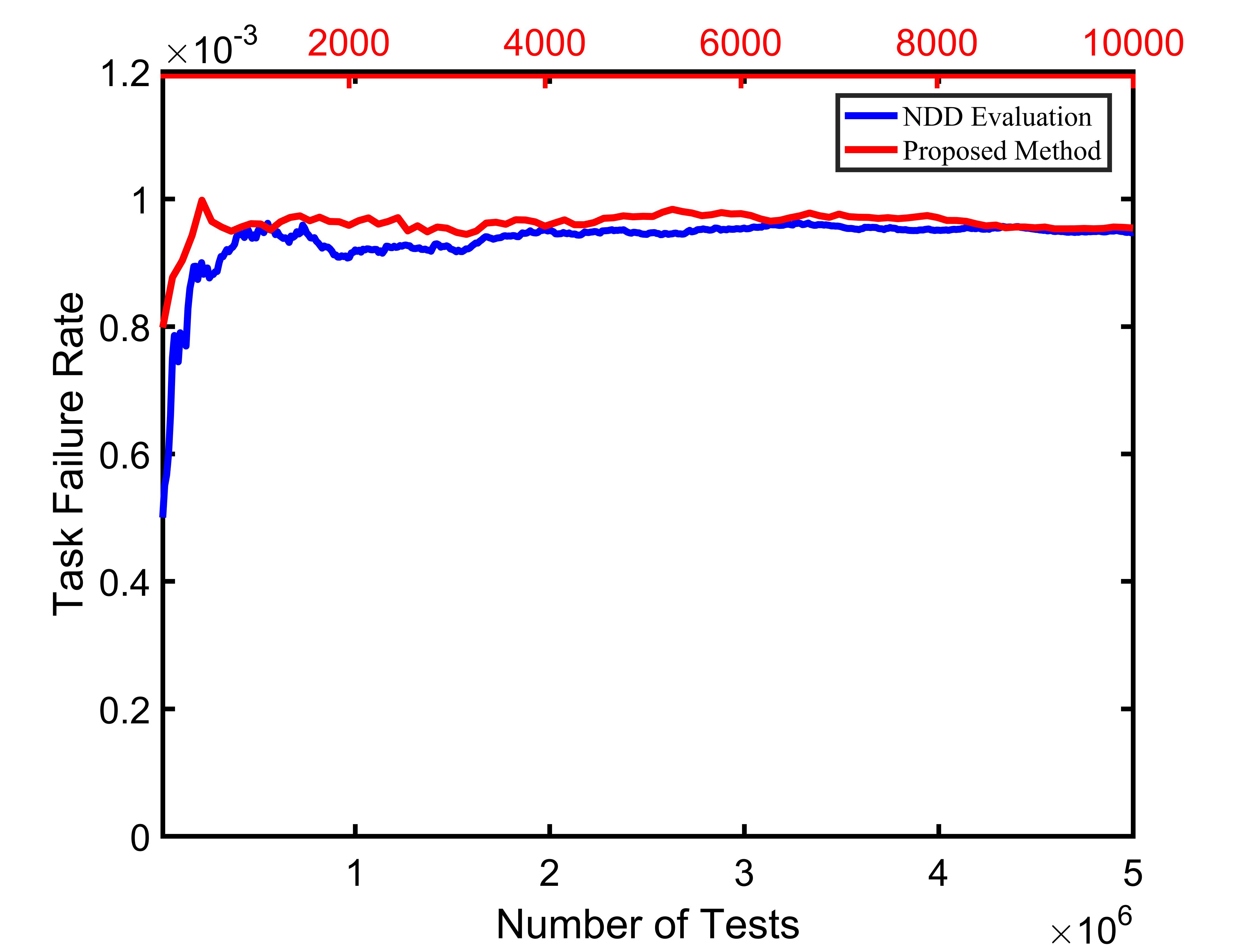}
		\centerline{(a)}
	\end{minipage}
	\begin{minipage}{.9\linewidth}
		\includegraphics[width=1\textwidth]{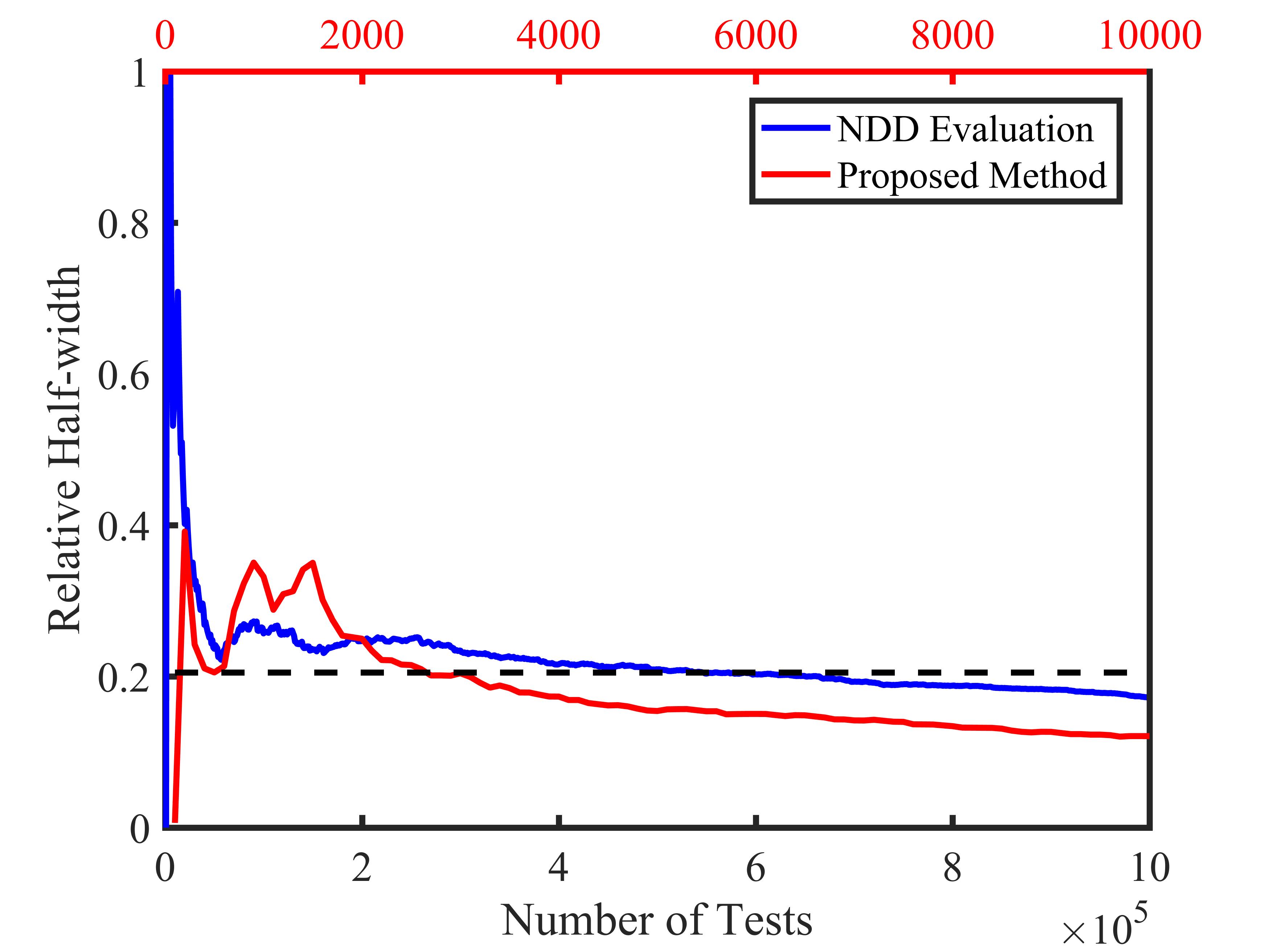}
		\centerline{(b)}
	\end{minipage}
	\caption{Results of the highway exit case: (a) estimation results of the task failure rate; (b) relative half-width of the estimation results.} 
	\label{fig_result_highway}
\end{figure}

Fig. \ref{fig_result_highway} shows the comparison of the two evaluation methods. The legends and axis are the same as those in the cut-in case. Similar with the previous case study, both methods can obtain unbiased estimation of the failure rate with the relative half-width ($\beta=0.2$). Fig. \ref{fig_result_highway} (b) shows that the proposed method achieves this estimation precision after $2.6\times 10^3$ tests, while the NDD evaluation method takes $6.6\times10^5$ tests. The proposed method is about 255 times faster than the NDD evaluation method. 

\section{Car-following Case Study}
 The car-following case is designed to show the ability of the proposed methods in solving the TSLG problem with high-dimensionality. As shown in Fig. \ref{fig_Case} (c), the test CAV follows a BV for a certain period of time. The decision variables include the initial condition and acceleration profile of the leading BV: 
\begin{eqnarray}
\label{eq_x_CF}
x = \left[ v_0, R_0,  \dot{R}_0, u_1, u_2, \dots, u_m \right]^T, x\in\B{X}
\end{eqnarray}
where $v_0$ denotes the initial velocity of the leading BV, $R_0$ and $\dot{R}_0$ denote the initial range and range rate between the BV and CAV, $m$ denotes the total time steps, and $u_1, u_2, \dots, u_m$ denote the acceleration sequences of the BV. If the BV is controlled every $1s$, for a $30s$ car-following scenario, the dimension of the scenario is 33. Since the computation complexity grows exponentially with the dimension, the problem faces ``curse of dimensionality''. To the best of our knowledge, none of the existing evaluation methods can be applied to evaluate the high-dimensional cases. 

The key to handle high-dimensionality is to formulate the TSLG problem as a Markov Decision Process (MDP) problem. 
Let $s=(v_{BV}, R, \dot{R}) \in \C{X}$ denote the state, where $v_{BV}$ denotes the speed of the BV, $R$ is the range, $\dot{R}$ is the range rate, and $\C{X}$ is the feasible set of states. Let $u \in \B{U}$ denote the action, where $\B{U}$ is the feasible acceleration set of the leading BV. It is assumed that the Markovian property holds considering the next action is dependent only on the current state, i.e., the acceleration of the BV is dependent only on its current speed. Then a testing scenario $x$ in Eq. (\ref{eq_x_CF}) can be described as a series of states and actions  (i.e., $s_1 \overset{u_1}{\to} s_2 \overset{u_2}{\to} \dots $), and the set of feasible scenarios $\B{X}$ can be represented by the decision tree, as shown in Fig. \ref{fig_tree}. Every branch from the initial state to the terminal state (i.e., leaf node) specifies a testing scenario. The library generation problem now turns into the problem of finding the critical branches for CAV evaluation. 

\begin{figure}[h!]
	\centering
	\includegraphics[width=0.45\textwidth]{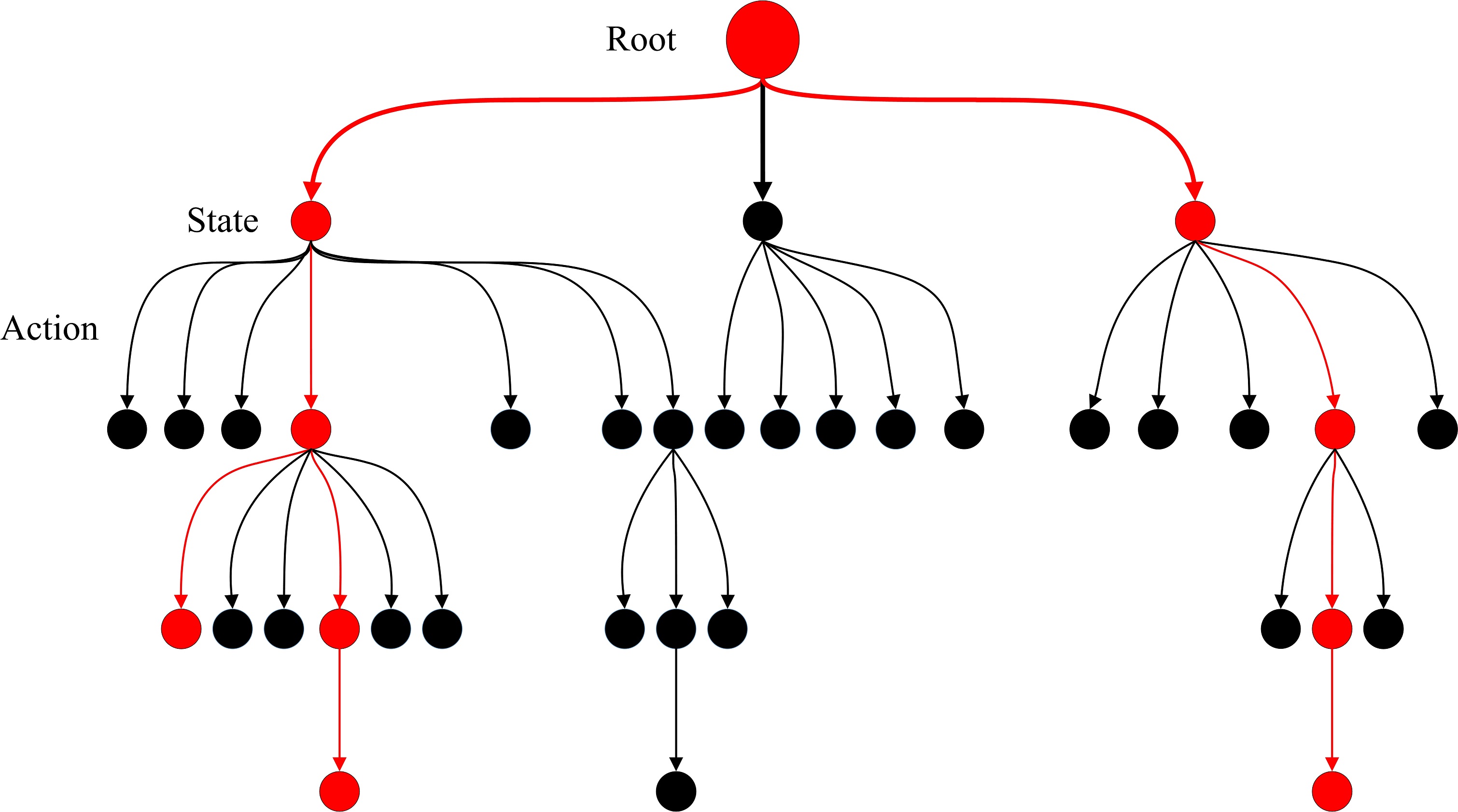}
	\caption{Illustration of the generated library (i.e., red circles and red arrows) in the scenario space.} 
	\label{fig_tree}
\end{figure}

In order to compute the criticality of each scenario, here we define the value of each state-action pair  $Q(s, u)$ as
\begin{eqnarray}
\label{eq_Q}
Q(s_k, u_k) =  P(S|u_k, s_k) P(u_k|s_k), k=1,\cdots,m.
\end{eqnarray}
The definition of $Q$ is consistent with the proposed definition of criticality. As shown in Eq. (\ref{eq_Q}), the left term $P(S|u_k, s_k)$ represents ``the probability of the event $S$ if the scenario is currently at the state $s_k$ and take the action $u_k$'', which measures the maneuver challenge. The right term $P(u_k|s_k)$ represents ``the probability of taking action $u_k$ if the scenario is currently at the state $s_k$'', which is the exposure frequency. With such definition, we can prove that the scenario criticality can be computed as 
\begin{eqnarray}
V(x) = C(x) \prod_{k=1}^{m} Q(s_k, u_k), \nonumber
\end{eqnarray}
where $C(x)$ is a normalization factor of the scenario $x$, and $\theta$ is omitted to simplify notations. The proof of this equation can be found in Theorem 1. 

To obtain $Q(s, u)$, the temporal-difference (TD) reinforcement learning (RL) technique is applied. BVs are the ``agent'' of the RL scheme, and the state of test CAV is the ``environment'', which is represented by a deterministic SM. The state transition is influenced by both BVs and the deterministic SM.
The TD-RL method updates the $Q(s_k, u_k)$ based on the estimation of the next state value $Q(s_{k+1}, u_{k+1})$ (i.e., the TD(0) method in \cite{sutton2018reinforcement}). The iterative update is based on the TD error, which measures the difference between the current estimation of $Q(s_k, u_k)$ and the new estimation. Let $\delta_k$ denote the TD error at the time step $k$, then an iteration equation can be obtained as
\begin{eqnarray}
\label{eq_up_pre}
Q(s_{k+1}, u_{k+1}) \leftarrow  Q(s_k, u_k) + \alpha \delta_k, \nonumber
\end{eqnarray}
where $\alpha$ is the learning rate, e.g., 0.1.
In Theorem 2, we prove that, after the training process of TD-RL, $Q(s, u)$ can converge to the values defined in Eq. (\ref{eq_Q}) if the TD error is defined as
\begin{eqnarray}
\label{eq_TDerr}
\delta_k = \left(
\sum_{u_{k+1}\in\B{U}} Q(s_{k+1}, u_{k+1})
\right) P(u_k|s_k) - Q(s_k, u_k).
\end{eqnarray}
By pruning the uncritical state-action pairs of the decision tree, the critical scenario library contains all branches with $V(x|\theta)>0$. As shown in Fig. \ref{fig_tree}, the branches of consecutive red nodes and arrows represent the critical scenarios.

\subsection{Problem formulation for the car-following case}
In the car-following case, safety is selected as the performance metric and accident rate is used to represent safety. The scenario state contains three variables, i.e., speed of the leading BV ($v_{BV}$), range ($R$) between the leading BV and the test CAV, and the range rate ($\dot{R}$):
\begin{equation}
s =(v_{BV}, R, \dot{R}) \in \C{X}.
\end{equation}
The leading BV's acceleration ($u$) is defined as the action. We discretise the range ($R\in(0,115]$), range rate ($\dot{R}\in [-10,8]$), velocity ($v\in[20,40]$), and acceleration ($u\in[-4,2]$) by $1m$, $1m/s$, $1m/s$, and $0.2m^2/s$ respectively. The leading BV is controlled every $1s$. For a $30s$ car-following case, the size of the entire scenario space is $N(\B{X})=21 \times 115 \times 19 \times 31^{30}$. The size of the entire state space is $N(\C{X}) = 21 \times 115 \times 19 = 45,885$, and the size of the entire state-action space is $N(\C{X})×N(\B{U})=45,885\times31 \approx 1.4 \times 10^5$, both of which are much smaller than the entire scenario space $N(\B{X})$.


\subsection{Library Generation}
The same NDD of the highway exit case is used in the car-following case, where both the car-following and free-driving events are extracted. The car-following events are utilized to calculate the exposure frequency of states $P(s)$, while the free-driving events are utilized to estimate the exposure frequency of actions $P(u|s)$. 

To improve the searching efficiency of critical scenarios, the state space is classified into three zones, i.e., collision zone, dangerous zone, and safe zone. The collision zone is defined by the states $\C{X}_c = \{ s\in \C{X}| R \le d_{acci} \}$, where $d_{acci}$ is a distance threshold for an accident, e.g., $1m$. The safe zone $\C{X}_s$ is defined by the states which cannot lead to an accident with even the most extreme actions of the leading BV, i.e., BV decelerates with maximal deceleration. The dangerous zone $\C{X}_d$ contains the states which have probabilities leading to an accident. Then values of $P(S|u_k, s_k )$ for states in different zones can be obtained as
\begin{equation}
P(S|u_k, s_k) = \left\{ 
\begin{matrix}
0, & s_k \in \C{X}_s \\
1, &s_k \in \C{X}_c
\end{matrix}.\nonumber
\right.
\end{equation}
As shown in Fig. \ref{fig_graph}, a non-trivial car-following testing scenario should start from a dangerous state (i.e., root state) and stop at a collision state or a safe state (i.e., terminal state).  The critical scenario should contain the state in the collision zone. The same car-following SM is applied as in Eq. (\ref{eq_SM_ct}).  By simulations of the SM, the dangerous zone is obtained, which consists about 5,000 states (10\% of all the states).

\begin{figure}[h!]
	\centering
	\includegraphics[width=0.45\textwidth]{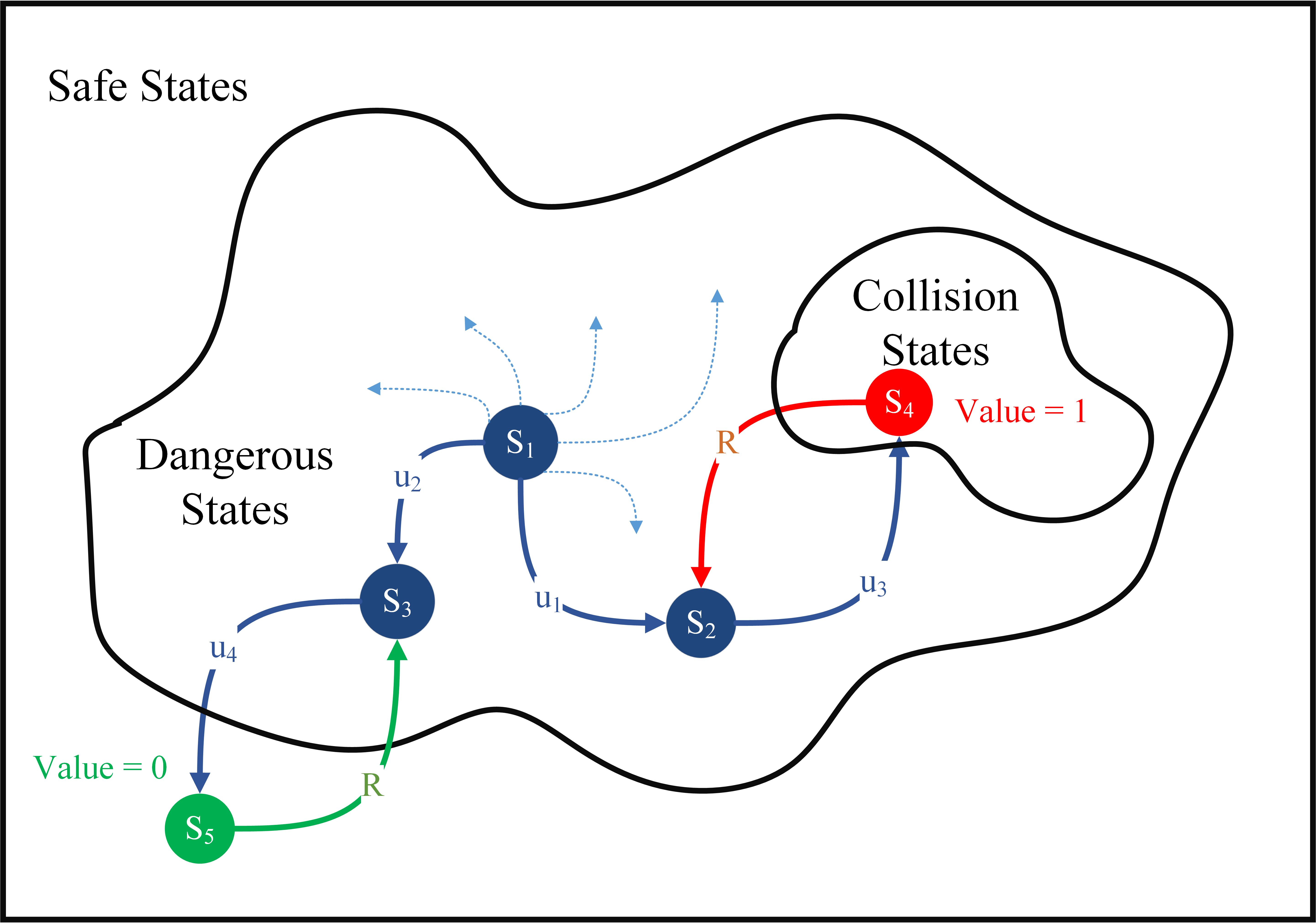}
	\caption{Illustration of the state transitions for car-following scenarios. States are classified into three zones, i.e., collision zone, dangerous zone, and safe zone.} 
	\label{fig_graph}
\end{figure}

The initial $Q$ values are obtained from the NDD. To improve the training efficiency, a uniform distribution is applied as the training policy, which guarantees that all state-action pairs can be visited for unlimited number of times if the training has not stopped \cite{sutton2018reinforcement}. The absolute value of the TD error is defined as the stop criteria as $\abs{\delta_t}< \delta_0$, where $\delta_0$ is a pre-determined threshold, e.g., $10^{-10}$.

\begin{figure}[h!]
	\centering
	\begin{minipage}{.9\linewidth}
		\includegraphics[width=1\textwidth]{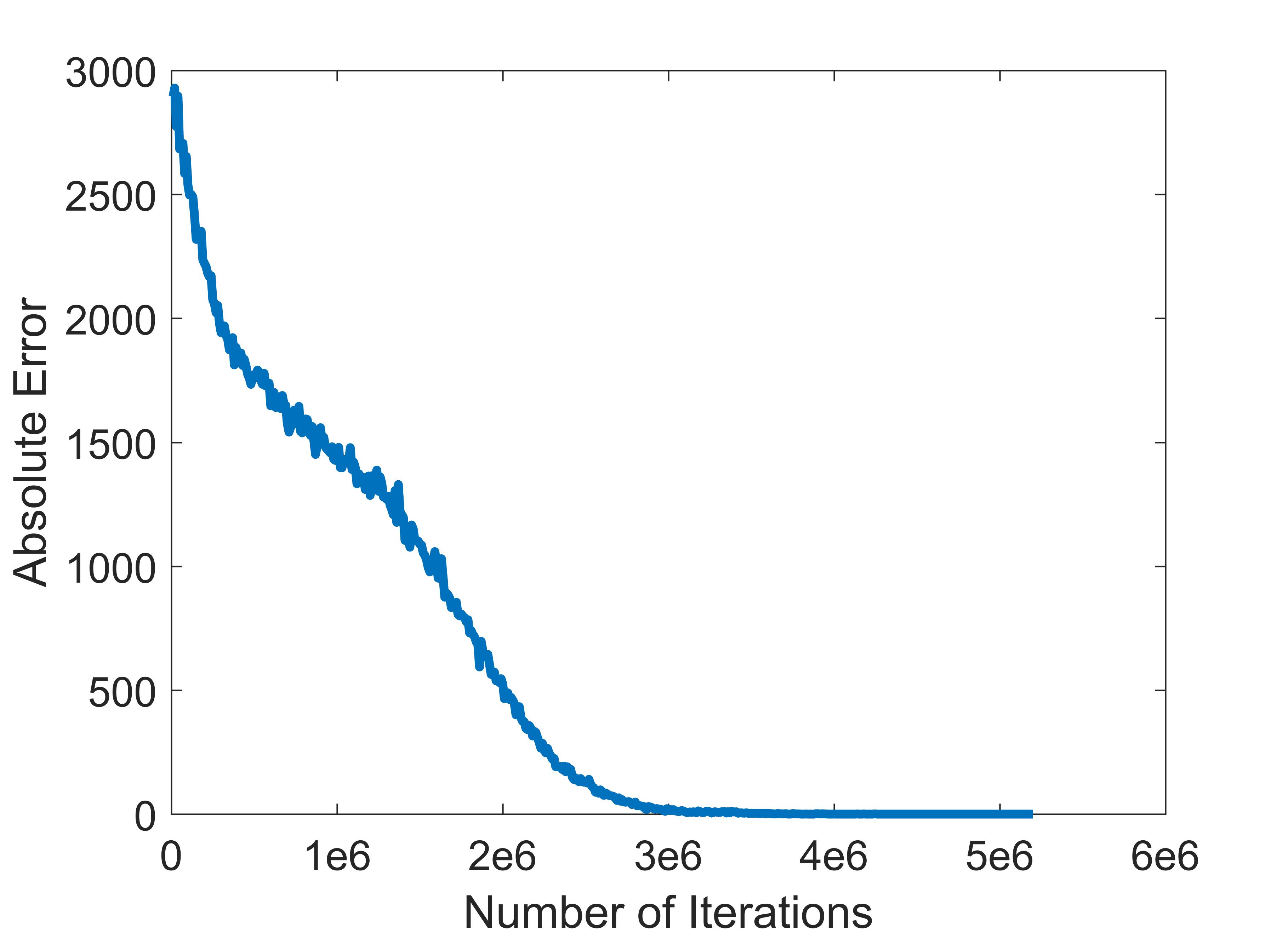}
		\centerline{(a)}
	\end{minipage}
	\begin{minipage}{.9\linewidth}
		\includegraphics[width=1\textwidth]{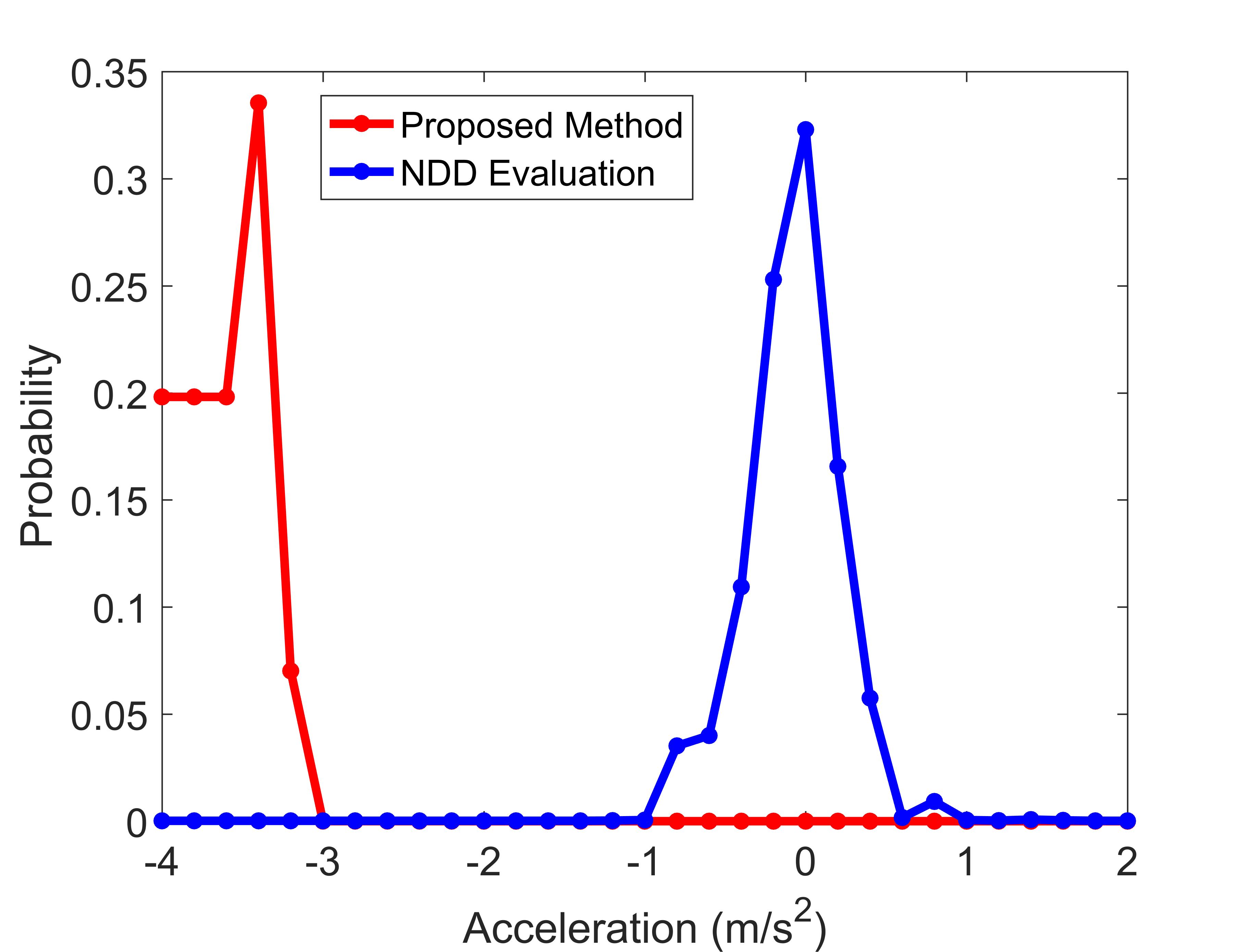}
		\centerline{(b)}
	\end{minipage}
	\caption{(a) The training results of the absolute TD error and (b) probability distribution of the actions at the state $s=(38,6,-2)$.} 
	\label{fig_eg_CF}
\end{figure}

The training is conducted with Matlab 2018, in a workstation equipped with Intel i7-7700 CPU and 16G RAM. It takes about 20 minutes to reach convergence. Fig. \ref{fig_eg_CF} (a) shows the convergence of the absolute TD error with learning iterations. The values of  state-action pairs converge after about $3\times10^6$ steps of iterations. Fig. \ref{fig_eg_CF} (b) shows an example of the probability distribution of the actions for a dangerous state, i.e., $s=(38,6,-2)$. The distribution from NDD is represented as the blue line (i.e., $P(u|s)$), while the generated distribution by the RL-enhanced method is represented as the red line (i.e., $P(u|s, S)$). It shows that the generated distribution behaves more aggressively than NDD with higher probabilities at extreme decelerations. The highest probability lies in $u=-3.4$ $m/s^2$, instead of $u=-4$ $m/s^2$, which is consistent with the proposed definition of criticality combined of both maneuver challenge and exposure frequency.

\subsection{CAV Evaluation}
After the above steps, the testing scenario library of the car-following case is generated. Testing scenarios can be sampled from the scenario library. The initial state is generated by Eq. (\ref{eq_PxS_decom1}), and accelerations of the BV are generated by Eq. (\ref{eq_PxS_decom3}). Similar to the previous cases, the $\epsilon$-greedy sampling policy is applied in the sampling process with $\epsilon=0.1$. As shown by the red line in Fig. \ref{fig_eg_CF} (b), the probability of acceleration greater than -3 is zero, i.e., out of the library. By adopting the $\epsilon$-greedy policy, however, these acceleration values can be sampled with a small probability. Similarly, the initial state has a small probability to be sampled from the safe states as well. The same CAV car-following model used in the cut-in case study \cite{zhao2017accelerated} is evaluated with the generated library. The NDD evaluation method is used as the baseline.

\begin{figure}[h!]
	\centering
	\begin{minipage}{.9\linewidth}
		\includegraphics[width=1\textwidth]{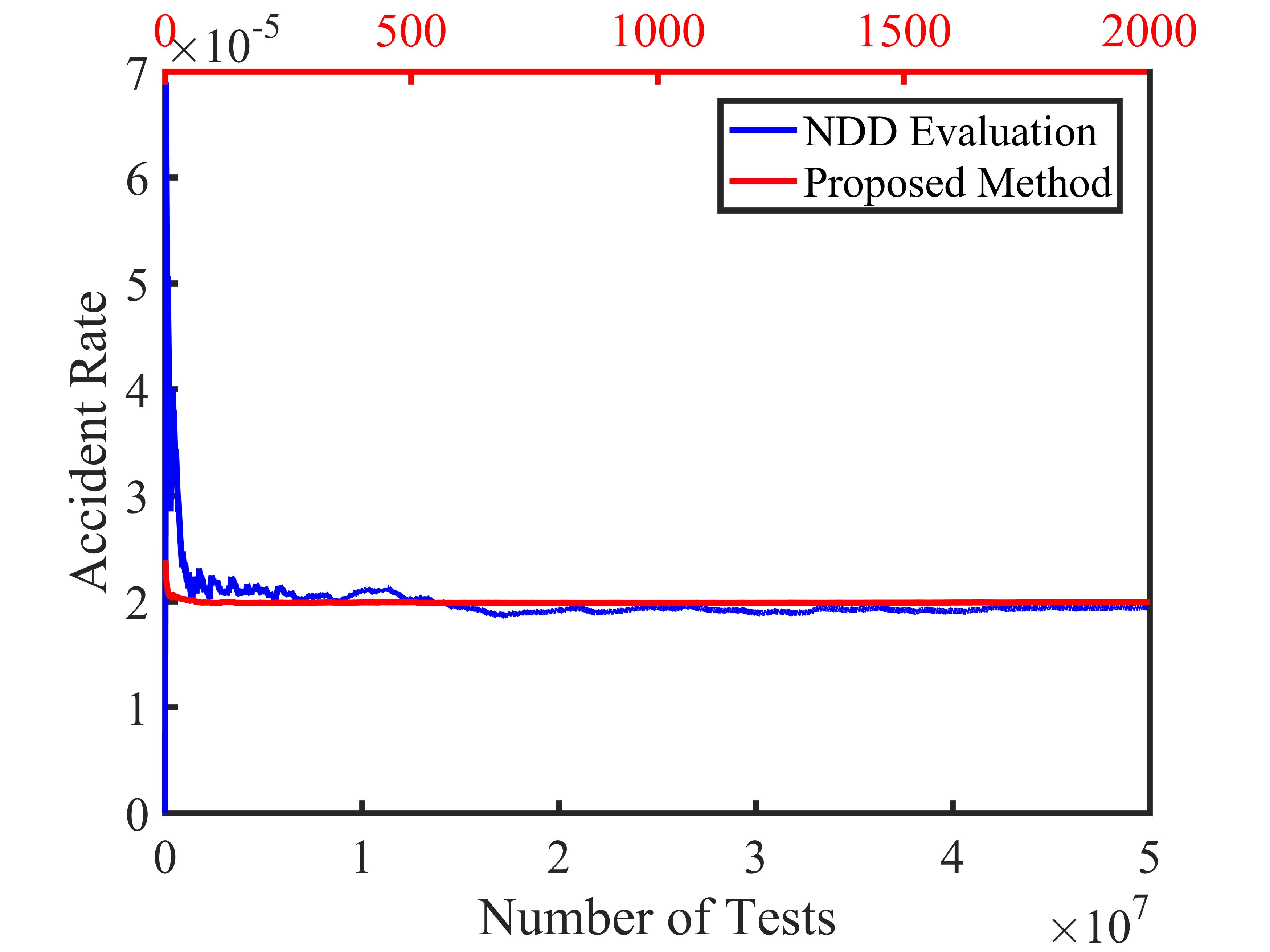}
		\centerline{(a)}
	\end{minipage}
	\begin{minipage}{.9\linewidth}
		\includegraphics[width=1\textwidth]{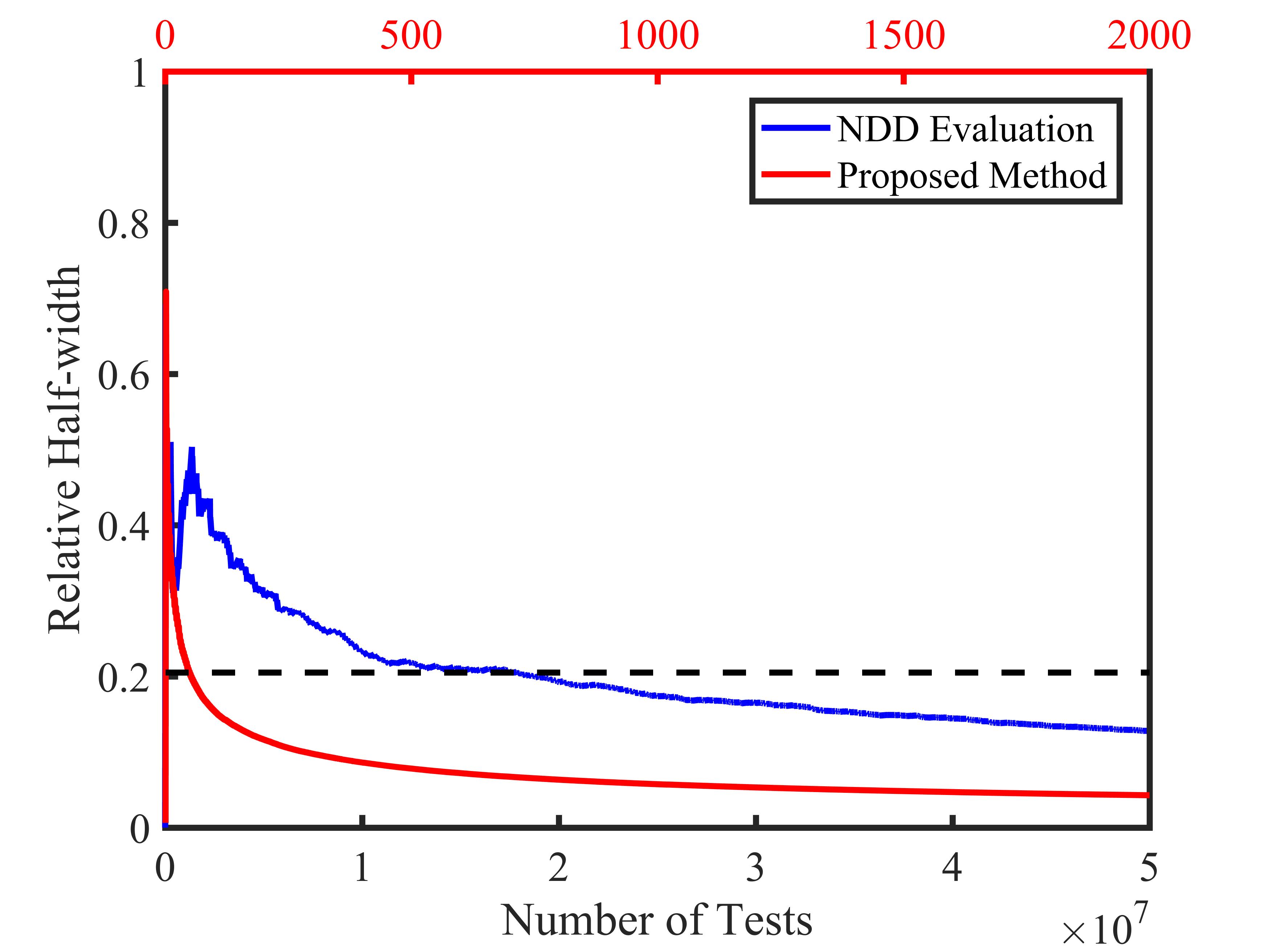}
		\centerline{(b)}
	\end{minipage}
	\caption{Results of the car-following case: (a) estimation results of the accident rate; (b) relative half-width of the estimation results.} 
	\label{fig_result_CF}
\end{figure}

Fig. \ref{fig_result_CF} shows comparison of the two evaluation methods. The blue line represents results of the NDD evaluation method, and the red line represents results of the proposed method. As shown in Fig. \ref{fig_result_CF}, both methods can obtain accurate estimation of accident rate with the same estimation precision ($\beta=0.2$). Fig. \ref{fig_result_CF} (b) shows that the proposed method achieves this estimation precision after $50$ tests, while the NDD evaluation method takes $1.875\times10^7$ tests. The proposed method is about $3.75\times10^5$ times faster than the NDD evaluation method.

\section{Discussions}
In this section, based on the results of the three case studies, the advantages and limitations of the proposed method are discussed.

\subsection{Advantages of the proposed method}

As demonstrated by the case studies, the proposed method is generic as it can be applied for evaluation with different performance metrics (e.g., safety and functionality) and varying CAV models under different ODDs. It can accelerate the CAV testing for both low and high dimensional scenarios.

As demonstrated by all three cases, to reach required evaluation accuracy, the proposed method can significantly reduce the number of tests, comparing with the on-road test method. Because the most time-consuming and expensive step in the CAV evaluation process is expected to be vehicle testing, the proposed method can significantly save cost. Similarly, the method can also be applied in simulation platforms with the same advantage, because testing CAVs in high-fidelity simulations is also the most time-consuming step.

Statistical coverage of scenarios is guaranteed by the proposed method, as all scenarios have possibilities to be tested. It utilizes more domain knowledge (e.g., scenario criticality) and outperforms the enumerative coverage (e.g., enumeration of all possible scenarios with a certain resolution), which suffers from the ``curse of dimensionality''. 

In addition, the performance index, $P(A|\theta)$, can quantitatively and interpretively measure the performance of CAVs in naturalistic driving environment. Taking safety evaluation as an example, the accident rate is the most natural index to evaluate the safety performance. 

\subsection{Limitations of the proposed method}

As indicated in Theorem 2 in the Part I paper, the efficiency of the proposed method is affected by the ``dissimilarity'' between the SM and the CAV under test, i.e., $P(A|x, \theta) - P(S|x, \theta)$. This is the major reason why the acceleration effects of the three cases are different, from 255 to $3.75\times10^5$ times. In this paper, the commonly-used human driving models (such as the IDM in the cut-in case and car-following case) were adopted as SMs, which perform reasonably well and can be served as a starting point for testing library generation. To address the dissimilarity issue, adaptive modification of the SM for different CAVs deserves more investigation. For CAV developers, this problem can be solved naturally as the CAV model can be used directly for scenario generation.

Large-scale NDD is required to calculate the exposure frequency of scenarios, which could be a limitation to apply the proposed method. With the deployment of on-board and infrastructure-based sensors, however, large-scale NDD can be collected with lower cost and become more accessible. For example, both research institutes \cite{wang2017much} and companies (e.g., Waymo, Tesla, Mobileye, \emph{etc.}) are collecting such data.

\section{Conclusions}
 This paper complements the general TSLG methodology developed in the Part I paper by providing three example case studies, i.e., cut-in, highway exit, and car-following. More importantly, the proposed method in Part I was enhanced by a temporal-difference reinforcement learning (TD-RL) method to generate high-dimensional scenarios efficiently. For all three cases, our results show that the proposed method can effectively and efficiently generate the testing scenario library, which can accelerate the evaluation process by 255 to $3.75\times10^5$ times compared with the NDD evaluation method, but with the same accuracy.

Combining Part I and Part II papers, to the best of our knowledge, this is the first study that provides a systematic framework and implementation guidelines for both low and high dimensional scenarios, different performance metrics, and varying CAV models. 

There are many interesting topics that can be further investigated. 
For the proposed TSLG method, since the dissimilarity between the SM and the test CAV is the major cause of evaluation inefficiency, SM as well as the generated library can be updated adaptively using the data collected from the testing process. Another important topic is to investigate the scenario generation method for extremely high dimensions, for example, urban driving environment with hundreds of background vehicles. As both the spatial and temporal complexity will increase greatly, RL-enhanced method presented in this paper needs further improvement. These topics are left for future studies.

\appendices

\section{Proof of Theorems}

\begin{myTheo}
	If $Q(s, u)$ is defined as Eq. (\ref{eq_Q}) and $x$ is defined as Eq. (\ref{eq_x_CF}), the scenario criticality can be computed as
	\begin{eqnarray}
	V(x) = C(x) \prod_{k=1}^{m} Q(s_k, u_k), \nonumber
	\end{eqnarray}
	where
	\begin{eqnarray}
	C(x) = \frac{P(S)} {  \sum_{s_1 \in \C{X}} \left( P(s_1) \cdot  \prod_{k=1}^{m} \left( \sum_{u_k \in \B{U}} Q(s_k, u_k)\right) \right) }. \nonumber
	\end{eqnarray}
\end{myTheo}

\begin{proof}
	After scenarios are represented by the decision tree, the exposure frequency of a testing scenario can be denoted as
	\begin{eqnarray}
	\label{eq_Px_Sim}
	P(x) = P(s_1) P(u_1|s_1) \dots P(u_m | s_m).
	\end{eqnarray}
	As shown in Eq. (\ref{eq_Value}), the criticality of a scenario is defined as
	\begin{eqnarray}
	\label{eq_Vx}
	V(x) = P(S|x) P(x) =  P(S) P(x|S),
	\end{eqnarray}
	where $P(S)$ is a constant, which can be obtained by Monte Carlo simulation.
	Similar to the Eq. (\ref{eq_Px_Sim}), $P(x|S)$ is denoted as
	\begin{eqnarray}
	\label{eq_PxS_Sim}
	P(x|S) = P(s_1|S) P(u_1|s_1,S) \dots P(u_m | s_m, S).
	\end{eqnarray}
	By applying Bayesian equation and Law of total probability, we have
	\begin{eqnarray}
	\label{eq_PxS_decom1}
	P(s_1|S) = \frac{P(S|s_1) P(s_1)}{\sum_{s_1\in \C{X}} P(S|s_1)P(s_1)},
	\end{eqnarray}
	\begin{eqnarray}
	\label{eq_PxS_decom2}
	P(S|s_1) = \sum_{u_1\in \B{U}} P(S|u_1, s_1)P(u_1|s_1),
	\end{eqnarray}
	\begin{eqnarray}
	\label{eq_PxS_decom3}
	P(u_k|s_k, S) =  \frac{P(S|u_k, s_k) P(u_k|s_k)}{\sum_{u_k \in \B{U}} P(S|u_k, s_k) P(u_k|s_k)},
	\end{eqnarray}
	where $k =1,\cdots, m$. Substituting Eq. (\ref{eq_PxS_Sim}), (\ref{eq_PxS_decom1}), (\ref{eq_PxS_decom2}), and (\ref{eq_PxS_decom3})  into Eq. (\ref{eq_Vx}), the theorem is concluded.
\end{proof}

\begin{myTheo}
	After the training process of TD-RL, $Q(u,s)$ can converge to the values defined in Eq. (\ref{eq_Q}), if the TD error is defined as
	\begin{eqnarray}
	\delta_k = \left(
	\sum_{u_{k+1}\in\B{U}} Q(s_{k+1}, u_{k+1})
	\right) P(u_k|s_k) - Q(s_k, u_k). \nonumber
	\end{eqnarray}
\end{myTheo}
\begin{proof}
	The TD error is the difference between the estimated value of $Q(s_k, u_k)$ and its estimation from the next state. Therefore, if the first term on the right in Eq. (\ref{eq_TDerr}) is the estimation of $Q(s_k, u_k)$ based on the next state, the theorem can be proved \cite{sutton2018reinforcement}.
	Consider that
	\begin{eqnarray}
	P(S|u_k,s_k) &&= \sum_{s_{k+1}\in \C{X}}P(s_{k+1}|u_k, s_k) P(S|s_{k+1}), \nonumber \\
	&&=P(S|s_{k+1}), \nonumber \\
	&&= \sum_{u_{k+1}\in\B{U}} P(S|u_{k+1}, s_{k+1}) P(u_{k+1}|s_{k+1}), \nonumber\\
	&&=\sum_{u_{k+1}\in\B{U}}Q(s_{k+1}, u_{k+1}), \nonumber
	\end{eqnarray}
	{where the second equivalence is derived considering the SM is deterministic.} Plugging the equation into Eq. (\ref{eq_Q}), we obtain
	\begin{eqnarray}
	Q(s_k, u_k) =  \left(\sum_{u_{k+1}\in\B{U}}Q(s_{k+1}, u_{k+1})\right) P(u_k|s_k), \nonumber
	\end{eqnarray}
	which concludes the theorem.
\end{proof}

\appendices

\bibliographystyle{IEEEtran}
\bibliography{IEEEexample}

\begin{IEEEbiography}[{\includegraphics[width=1in,height=1.25in,clip,keepaspectratio]{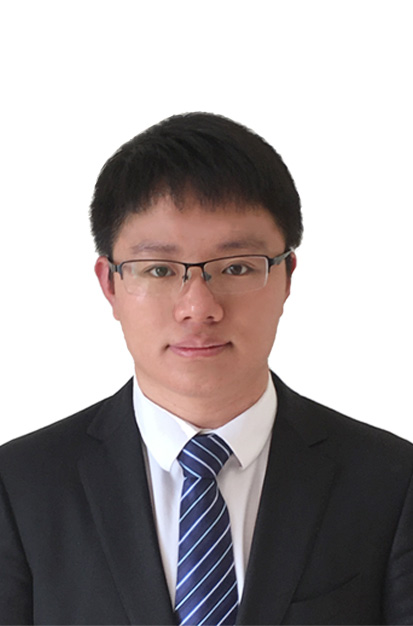}}]{Shuo Feng}
	received the bachelor’s degree and Ph.D. degree in Department of Automation from Tsinghua University, China, in 2014 and 2019. He was also a joint Ph.D. student in Civil and Environmental Engineering in University of Michigan, Ann Arbor, from 2017 to 2019. He is currently a postdoctoral researcher in Civil and Environmental Engineering in University of Michigan, Ann Arbor.	His current research interests include connected and automated vehicle evaluation, mixed traffic control, and transportation data analysis.
\end{IEEEbiography}

\begin{IEEEbiography}[{\includegraphics[width=1in,height=1.25in,clip,keepaspectratio]{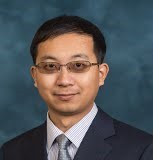}}]{Yiheng Feng}
	is currently an Assistant Research Scientist at University of Michigan Transportation Research Institute. He graduated from the University of Arizona with a Ph.D degree in Systems and Industrial Engineering in 2015. He has a Master degree from the Civil Engineering Department, University of Minnesota, Twin Cities in 2011. He also earned the B.S. and M.E. degree from the Department of Control Science and Engineering, Zhejiang University, Hangzhou, China in 2005 and 2007 respectively. His research interests include traffic signal systems control and security, and connected and automated vehicles testing and evaluation.
\end{IEEEbiography}

\begin{IEEEbiography}[{\includegraphics[width=1in, height=1.25in,clip,keepaspectratio]{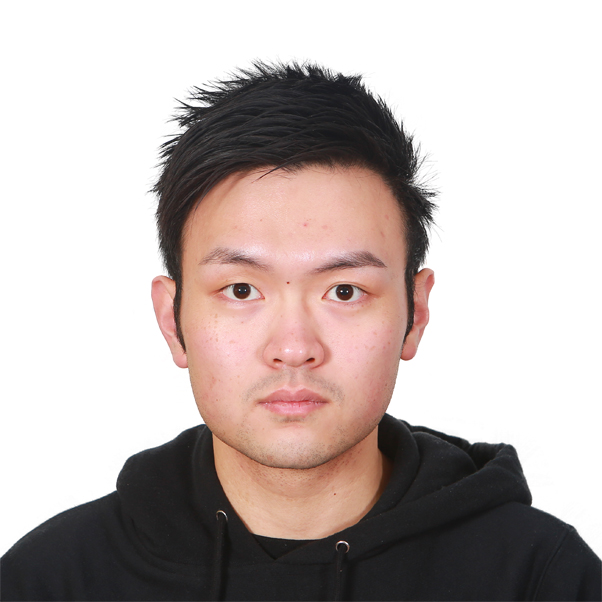}}]{Haowei Sun}
	is currently a graduate at Department of Civil and Environmental Engineering at University of Michigan. He received the bachelor’s degree in Department of Automation  from Tsinghua University, China, in 2019, and he visited University of Michigan for a summer research internship in 2018. His research interests include intelligent transportation, optimization method and deep reinforcement learning.
\end{IEEEbiography}

\begin{IEEEbiography}[{\includegraphics[width=1in,height=1.25in,clip,keepaspectratio]{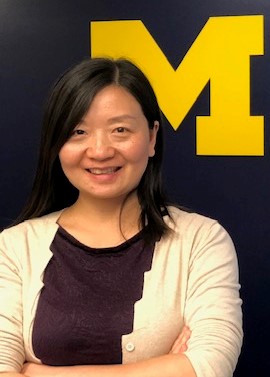}}]{Bao Shan}
	is an Associate Professor at the University of Michigan-Dearborn, and an Associate Research Scientist at the University of Michigan Transportation Research Institute, where she has been conducting research on human factors issues related to transportation systems. She received her Ph.D. in Mechanical and Industrial Engineering from University of Iowa in 2009. Her current research interests focus on human factors issues related to connected and automated vehicle technologies. 
\end{IEEEbiography}

\begin{IEEEbiography}[{\includegraphics[width=1in,height=1.25in,clip,keepaspectratio]{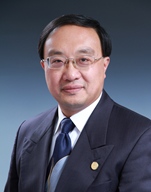}}]{Yi  Zhang}
	received   the   BS   degree   in1986   and   MS   degree   in   1988  from Tsinghua University in China, and earned  the  Ph.D.  degree  in  1995  from  the University  of  Strathclyde  in  UK.  He  is a   professor   in   the   control   science   and engineering  at  Tsinghua  University  with his  current  research  interests  focusing  on intelligent  transportation  systems. His  active  research  areas include  intelligent  vehicle-infrastructure  cooperative  systems, analysis  of  urban  transportation  systems,  urban  road  network management,  traffic  data  fusion  and  dissemination,  and  urban traffic control and management.   His research fields also cover the advanced control theory and applications, advanced detection and measurement, systems engineering, etc.
\end{IEEEbiography}

\begin{IEEEbiography}[{\includegraphics[width=1in,height=1.25in,clip,keepaspectratio]{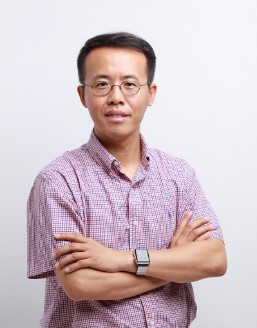}}]{Henry X. Liu}
	is a Professor of Civil and Environmental Engineering at the University of Michigan, Ann Arbor and a Research Professor of the University of Michigan Transportation Research Institute. He also directs the USDOT Region 5 Center for Connected and Automated Transportation. Dr. Liu received his Ph.D. degree in Civil and Environmental Engineering from the University of Wisconsin at Madison in 2000 and his Bachelor degree in Automotive Engineering from Tsinghua University in 1993. Dr. Liu's research interests focus on transportation network monitoring, modeling, and control, as well as mobility and safety applications with connected and automated vehicles. On these topics, he has published more than 100 refereed journal articles. Dr. Liu is the managing editor of Journal of Intelligent Transportation Systems and an associate editor of Transportation Research Part C. 
\end{IEEEbiography}

\end{document}